\documentclass[final,12pt]{colt2024} 

\title[Topological Expressivity of ReLU NNs]{Topological Expressivity of ReLU Neural Networks}
\usepackage{subcaption}
\usepackage{times}

\coltauthor{\Name{Ekin Ergen} \Email{ergen@math.tu-berlin.de}
\and
\Name{Moritz Grillo} \Email{grillo@math.tu-berlin.de}\\
\addr Technische Universit\"at Berlin, Berlin, Germany}

\usepackage[utf8]{inputenc}
\usepackage{amsfonts}
\usepackage{enumitem}
\usepackage{forloop}
\usepackage{cleveref}
\usepackage{mathtools}
\usepackage{notoccite}
\usepackage{pgfmath}
\usepackage{tikz}
\usepackage{tikz-cd}
\usetikzlibrary{intersections}
\usepackage{xcolor}
\usepackage{thmtools}
\usepackage{thm-restate}

\numberwithin{subcase}{case}
\newlist{casesp}{enumerate}{3} 
\setlist[casesp]{align=left, 
                 listparindent=\parindent, 
                 parsep=\parskip, 
                 font=\normalfont\bfseries, 
                 leftmargin=0pt, 
                 labelwidth=0pt, 
                 itemindent=.4em,labelsep=.4em, 
                 partopsep=0pt, 
                 }
\setlist[casesp,1]{label=Case~\arabic*:,ref=\arabic*}
\setlist[casesp,2]{label=Case~\thecasespi.\roman*:,ref=\thecasespi.\roman*}
\setlist[casesp,3]{label=Case~\thecasespii.\alph*:,ref=\thecasespii.\alph*}
\definecolor{light blue}{rgb}{0.68, 0.85, 0.9} 
\hypersetup{
	colorlinks=true,
	linkcolor=blue}

\newcommand{\altfrac}[2]{\ifmmode\def\tmp{$}\else\def\tmp{}\fi\mbox{%
    {\footnotesize \raisebox{.24\ht\strutbox}{\tmp#1\tmp}}%
    \kern-2.2pt\scalebox{1.6}[1]{/}\kern-1.8pt
    {\footnotesize\tmp#2\tmp}%
    }}

\newtheorem{def1}[theorem]{Definition}
\newtheorem{prop}[theorem]{Proposition}
\newtheorem{cor}[theorem]{Corollary}

\newtheorem{obs}[theorem]{Observation}
\newtheorem*{claim}{Claim}

\newcommand{\N}{\mathbb{N}}
\newcommand{\Po}{\mathcal{P}}
\newcommand{\R}{\mathbb{R}}
\newcommand{\Z}{\mathbb{Z}}
\newcommand{\1}{\mathbf{1}^T}
\newcommand{\sgn}{\operatorname{sgn}}
\newcommand{\Int}{\operatorname{int}}

\newcommand{\im}{\operatorname{im}}

\newcommand{\aff}{\operatorname{Aff}}
\newcommand{\conv}{\operatorname{conv}}
\newcommand{\sk}{\operatorname{sk}}
\newcommand{\interior}[1]{%
  {\kern0pt#1}^{\mathrm{o}}%
}
\newcommand{\ReLU}{
	\begin{tikzpicture}
		\draw [line width=2pt] (-1.1ex,0) -- (0,0) -- (0.9ex,0.9ex);
	\end{tikzpicture}
}
\tikzset{connection/.style={ -{Stealth[scale=1.5]}}}
\tikzset{neuron/.style={thick, circle, draw, inner sep = 0ex, minimum width = 4.5ex}}

\begin{document}
\maketitle
\begin{abstract}%
    We study the expressivity of ReLU neural networks in the setting of a binary classification problem from a topological perspective. Recently, empirical studies showed that neural networks operate by changing topology, transforming a topologically complicated data set into a topologically simpler one as it passes through the layers. This topological simplification has been measured by Betti numbers, which are algebraic invariants of a topological space. We use the same measure to establish lower and upper bounds on the topological simplification a ReLU neural network can achieve with a given architecture. We therefore contribute to a better understanding of the expressivity of ReLU neural networks in the context of binary classification problems by shedding light on their ability to capture the underlying topological structure of the data. In particular the results show that deep ReLU neural networks are exponentially more powerful than shallow ones in terms of topological simplification. This provides a mathematically rigorous explanation why deeper networks are better equipped to handle complex and topologically rich data sets.%
\end{abstract}
\section{Introduction}
Neural networks are at the core of many AI applications. A crucial task when working with neural networks is selecting the appropriate architecture to effectively tackle a given problem. Therefore, it is of fundamental interest to understand the range of problems that can be solved by neural networks with a given architecture, i.e., its \emph{expressivity}. 

 In recent years,  many theoretical findings have shed light on the expressivity of neural networks. Universal approximation theorems \citep{Cybenko1989ApproximationBS,Hornik1991ApproximationCO} state that one hidden layer is already sufficient to approximate any continuous function with arbitrary accuracy. On the other hand, it is known that deep networks can represent more complex functions than their shallow counterparts (see e.g.,~\cite{pmlr-v49-telgarsky16, pmlr-v49-eldan16, Arora2016UnderstandingDN}).

The measure of expressivity of a neural network should always be related to the problem it has to solve. A common scenario in which neural networks are employed is the binary classification problem, where the network serves as a classifier for a binary labeled data set. Since topological data analysis has revealed that data often has nontrivial topology, it is important to consider the topological structure of the data when dealing with a binary classification problem.  \citet{Topology_of_NN} show through empirical methods that neural networks operate topologically, transforming a topologically complicated data set into a topologically simple one as it passes through the layers. Given a binary labeled data set, they assume that the positively labeled and the negatively labeled points are sampled from topological spaces $X_a$ and $X_b$ respectively that are entangled with each other in a nontrivial way. Their experiments show that a well-trained neural network gradually disentangles the topological spaces  until they are linearly separable in the end, i.e, the space $X_b$ is mapped to the positive real line and $X_a$ to the negative real line. From a theoretical point of view, it is of interest to determine the extent of ``topological change'' that can be achieved by neural networks of a particular architecture. The topological expressivity of a neural network can therefore be measured by the complexity of the most complex topological spaces it can separate and is directly related to the complexity of binary classification problems.
 
In this paper we investigate the topological expressivity of ReLU neural networks, which are one of the most commonly used types of neural networks \citep{pmlr-v15-glorot11a,GoodBengCour16}. A \emph{$(L+1)$-layer neural network (NN)} is defined by $L+1$ affine transformations $T_\ell \colon \R^{n_{\ell-1}} \to \R^{n_\ell}$, $x \mapsto A_\ell x +b_\ell$ 
    for $A_\ell \in \R^{n_{\ell-1} \times n_\ell}, b_\ell \in \R^{n_\ell}$ and $ \ell =1,\ldots,L+1$.
    The tuple $(n_0,n_1,\ldots,n_{L},n_{L+1})$ is called the \emph{architecture}, $L+1$ the \emph{depth}, $n_\ell$ the \emph{width of the $\ell$-layer}, $\max\{n_1,\ldots,n_{L}\}$ the \emph{width} of the NN
   and $\sum_{\ell=1}^{L}n_\ell$ the \emph{size} of the NN.
   The entries of $A_\ell$ and $b_\ell$ for $\ell=1,...,L+1$ are called \emph{weights} of the NN and the vector space of all possible weights is called the \emph{parameter space} of an architecture. A {ReLU neural network} computes the function \[F = T_{L+1} \circ \sigma_{n_L} \circ T_{L} \circ \cdots \circ \sigma_{n_1} \circ T_1,\] where  $\sigma_n \colon \mathbb{R}^n \to \mathbb{R}^n$ is the \emph{ReLU function} given by \mbox{$\sigma_n(x)=(\max(0,x_1),\ldots,\max(0,x_n)).$}
   
    Note that the function $F$ is piecewise linear and continuous. In fact, it is known that any continuous piecewise linear function $F$ can be computed by a ReLU neural network~\citep{Arora2016UnderstandingDN}.  However, for a fixed architecture $A$, the class $\mathcal{F}_A$ of piecewise linear functions  that is representable by this architecture is not known \citep{hertrich2021towards,haase2023lower}.
    Conveniently, in the setting of a binary classification problem we are merely interested in the \emph{decision regions}, i.e., $F^{-1}((-\infty,0])$ and $F^{-1}((0,\infty))$ rather than the continuous piecewise linear function $F$ itself.
    
    {A common choice to measure the complexity of a topological space $X$ is the use of algebraic invariants. Homology groups are the essential algebraic structures with which topological data analysis analyzes data~\citep{tdabook} and hence Betti numbers as the ranks of these groups are a natural measure of topological expressivity. Intuitively, the $k$-th Betti number $\beta_k(X)$ corresponds to the number of $(k+1)$-dimensional holes in the space $X$ for $k>0$ and $\beta_0(X)$ corresponds to the number of path-connected components of $X$. Thus, one can argue that when a space (the support of one class of the data) has many connected components and higher dimensional holes, it is more difficult to separate this space from the rest of the ambient space, e.g., mapping it to the negative line. In Appendix \ref{top_background}, we present a brief introduction to homology groups. For an in-depth discussion of the aforementioned concepts, we refer to textbooks on algebraic topology (e.g.,~\cite{hatcher}).}

In order to properly separate $X_a$ and $X_b$, the sublevel set $F^{-1}((-\infty,0])$ of the function $F$ computed by the neural network should have the same topological complexity as $X_a$. 
\citet{Bianchini2014OnTC} measured the topological complexity of the decision region $F^{-1}((-\infty,0])$ with the sum of all its Betti numbers. This notion of topological expressivity does not differentiate between connected components and higher dimensional holes. 
On the other hand, if an architecture is not capable of expressing the Betti numbers of different dimensions of the underlying topological space of the data set, then for every $F \in \mathcal{F}_A$ there is a set of data points $U$ such that $F$ misclassifies every $x \in U$ \citep{Guss}. Therefore it is of fundamental interest to understand each Betti number of the decision regions, by which we define our notion of expressivity:
\begin{def1}
    The \emph{topological expressivity} of a ReLU neural network $F \colon \R^d \to \R$ is defined as the vector $\beta(F) = (\beta_k(F))_{k=0,\ldots,d-1}=(\beta_k(F^{-1}((-\infty,0]))_{k=0,\ldots,d-1}$.
\end{def1}

\subsection{Main Results}
Our main contribution consists of lower and upper bounds on the topological expressivity for ReLU NNs. These bounds demonstrate that  
    the growth of Betti numbers in neural networks depends on their depth. With an unbounded depth, Betti numbers in every dimension can exhibit exponential growth as the network size increases. However, in the case of a shallow neural network, where the depth remains constant, the Betti numbers of the sublevel set are polynomially bounded in size.
This implies that increasing the width of a network while keeping the depth constant prevents exponential growth in the Betti numbers.  Consequently, if a data set possesses exponentially high Betti numbers (parameterized by some parameter $p$), accurate modeling of the data set requires a deep neural network when the size of the neural network is constrained to be polynomial in parameter $p$ since the topological expressivity serves, as discussed above, as a bottleneck measure for effective data representation.

In Theorem~\ref{cor:exactformulaclosure}, the lower bounds for the topological expressivity are given by an explicit formula, from which we can derive the following asymptotic lower bounds:
    \begin{restatable}{cor}{asymptotic}
    \label{cor:main}
  Let $A=(d, n_1, \ldots, n_{L},1)$ with $n_{L}\geq 4d$ and $M=\prod_{\ell=1}^{L-1}2\cdot\left\lfloor\frac{n_\ell}{2d}\right\rfloor$, then there is a ReLU NN $F \colon \R^d \mapsto \R$ with architecture $A$ such that
  \begin{enumerate}[label=(\roman*)]
      \item $\beta_0(F)\in \Omega(M^d\cdot n_{L})$
      \item $\beta_k(F)\in \Omega(M^k\cdot n_{L})$ \ for \  $0<k < d$.

  \end{enumerate}
  In particular,  given $\textbf{v}=(v_1,\ldots, v_{d}) \in \N^{d-1}$, there is a ReLU NN $F$ of size $O\left(\log\left(\sum_{k=1}^{d-1}v_k\right)\right)$ such that
      $\beta_k(F) \geq v_{k+1}$ for all $k \in  \{0,\ldots, d-1\}$.
\end{restatable}
  Corollary~\ref{cor:main} provides a proof for a conjecture on lower bounds for the zeroth Betti number of the decision region given in the paper of~\citet{Guss}; in fact, it generalizes the statement to arbitrary dimensions.
  Furthermore, we observe that $L=2$ hidden layers are already sufficient to increase the topological expressivity as much as we want at the expense of an increased width due to the above lower bound.

\begin{cor}
    Given $v\in \N^d$, there exists an NN $F \colon \R^d \to \R$ of depth 2 such that $ \beta_k(F) \geq v_{k+1}$ for all $k \in  \{0,\ldots, d-1\}$.
\end{cor}
We obtain the lower bound by making choices for the weights of the NN,  nevertheless, we show that our construction is robust with respect to small perturbations. In fact, in Proposition~\ref{prop:robust_to_pertubation} we prove that we actually have an open set in the parameter space in which the respective functions all have the same topological expressivity.

Using an upper bound on the number of linear regions \citep{Bounding_Serra}, we obtain an explicit formula for an upper bound on $\beta_k(F)$ for an arbitrary architecture in Proposition~\ref{prop:upperbound}. This gives rise to the following asymptotic upper bounds:

\begin{restatable}{cor}{corupperbound}
\label{cor:main2}
     Let $F \colon \R^d \to \R$ be a neural network of architecture $(d,n_1,\ldots,n_L,1)$. Then it holds that
    $\beta_k(F) \in O\left(\left(\prod_{i=1}^Ln_i\right)^{d^2}\right)$ for $k \in [d-2]$ and $\beta_0,\beta_{d-1} \in O\left(\left(\prod_{i=1}^Ln_i\right)^{d}\right)$.
\end{restatable}

By combining Corollary~\ref{cor:main} and Corollary~\ref{cor:main2}, we can conclude that there is an exponential gap in the topological expressivity between shallow and deep neural networks. This aligns with other popular measures of expressivity, such as the number of linear regions, where similar exponential gaps are known \citep{Bounding_Serra,LinearRegions_Montufar,Notes_Montufar}.

\subsection{Related Work}
\subsubsection{Topology and Neural Networks}
Vast streams of research studying neural networks by means of topology using empirical methods \citep{petri2020on,Guss,Topology_of_NN,NEURIPS2020_5f146156}  as well as from a theoretical perspective  \citep{manifold_embedding,Melodia_2021,benthyperplane,Bianchini2014OnTC,grigsby2022local,topframework} have emerged recently. 
\citet{Bianchini2014OnTC} were the first that used Betti numbers as a complexity measure for decision regions of neural networks. 
Their work studies NNs with sigmoidal activation functions and shows that there is an exponential gap with respect to the sum of Betti numbers between deep neural networks and neural networks with one hidden layer. However, there are no insights about distinct Betti numbers. In \citet{Guss}, the decision regions of ReLU neural networks ares studied with empirical methods and an exponential gap for the zeroth Betti number is conjectured. Our results prove the conjecture and extend the results of \citet{Bianchini2014OnTC} for the ReLU case (see Section~\ref{Upperbound} and Appendix). 
Furthermore, topological characteristics such as connectivity or boundedness of the decision regions are also investigated in \citep{empirical_study,benthyperplane,grigsby2022local,wide_enough}.
\newline
 
\subsubsection{Expressivity of (ReLU) Neural Networks}
In addition to the universal approximation theorems \citep{Cybenko1989ApproximationBS,Hornik1991ApproximationCO}, there is a significant amount of research on the expressivity of neural networks, e.g., indicating that deep neural networks can be exponentially smaller in size than shallow ones. For ReLU neural networks, the number of linear regions is often used as a measure of complexity for the continuous piecewise linear (CPWL) function computed by the network. It is well established that deep ReLU neural networks can compute CPWL functions with exponentially more linear regions than shallow ones, based on various results such as lower and upper bounds on the number of linear regions for a given architecture \citep{Notes_Montufar,Bounding_Serra,LinearRegions_Montufar,Arora2016UnderstandingDN}. We partially use techniques from aforementioned works to establish our bounds on topological expressivity, which offers the advantage of being directly related to the complexity of binary classification problems.

\subsection{Notation and Definitions}
A function $F \colon \R^d \to \R^d$ is continuous piecewise linear (CPWL) if there is a polyhedral complex covering $\R^d$, such that $F$ is affine linear over each polyhedron of this complex. A linear region of $f$ is a maximal connected convex subspace $R$ such that $f$ is affine linear on $R$, i.e., a full-dimensional polyhedron of the complex.\footnote{In the literature there exists also a slightly different definition of a linear region leaving out the necessity of the region being convex, but the bounds we use are all applicable to this definition of a linear region.} 

For a survey on polyhedral theory in deep learning see \citet{huchette2023deep}, and for a general introduction to polyhedra we refer to \citet{Schrijver1986TheoryOL}.

    We denote by $[n]$ the set $\{1,\ldots,n\}$ and by $[n]_0$ the set $\{0,\ldots,n\}$.  We denote by $\pi_j \colon \R^d \to \R$ the projection onto the $j$-th component of $\R^d$ and by $p_j \colon \R^d \to \R^j$ the projection onto the first $j$ components.
    
    A crucial part of our construction is decomposing a unit cube into a varying number of small cubes. Thereby, given $\mathbf{m}=(m_1,\ldots, m_L) \in \N^L$ and $M = \left(\prod_{\ell=1}^Lm_\ell\right)$, the set $W^{(L,\mathbf{m},d)}_{i_1,\ldots,i_d}$ is defined as the cube of volume $\frac{1}{M^{d}}$ with ``upper right point'' $\frac{1}{M} \cdot (i_1,\ldots,i_d)$, i.e., the cube \mbox{$\prod_{k=1}^{d} [\frac{(i_k-1)}{M}, \frac{i_k}{M}] \subset [0,1]^d$}. The indices $(L,\mathbf{m},d)$ are omitted whenever they are clear from the context.

    {We denote by $D^k=\{x\in \R^k\colon \|x\|< 1\}$ the $k$-dimensional standard (open) disk and by \mbox{$S^k=\{x\in \R^{k+1}\colon \|x\|=1\}$} the $k$-dimensional standard sphere. We consider these sets as ``independent'' topological spaces. Therefore, it is justified to abstain from picking a specific norm, since all norms on $\R^k$ are equivalent.
    
    For $k,m\in \N$ with $m\leq k$, the \emph{($j$-dimensional open) $k$-annulus} is the product space $S^k\times D^{j-k}$.}
Note that since $S^k$ has one connected component and a $(k+1)$-dimensional hole, it holds that $\beta_0(S^k)=\beta_{k}(S^k)=1$ and the remaining Betti numbers equal zero.
The $j$-dimensional \mbox{$k$-annulus} is a $j$-dimensional manifold that can be thought as a thickened $k$-sphere and hence its Betti numbers coincide with the ones from the $k$-sphere. In Appendix \ref{top_background} the reader can find a more formal treatment of the latter fact.

    In contrast to $D^k$ and $S^k$, which are only seen as spaces equipped with a topology, we also consider neighborhoods around certain points $x\in \R^d$ as subsets of $\R^d$. To make a clear distinction, we define the space $B_r^d(x)$ as the \emph{$d$-dimensional open $r$-ball around $x$ with respect to the $1$-norm}, i.e., the space $\{y\in \R^d\colon \|x-y\|_1<r \}$. {Note that for $r<r'$, the set $B_r^k(x)\setminus \overline{B_{r'}^k(x)}$ is homeomorphic to a $k$-dimensional $(k-1)$-annulus and we will refer to them as $(k-1)$-annuli as well. These annuli will be the building blocks of our construction for the lower bound.}

{The rest of the paper is devoted to proving the lower and upper bounds. Most of the statements come with an explanation or an illustration. In addition, formal proofs for these statements are also provided in the appendix.}

\section{Lower Bound}
\label{lowerbound}

In this section, our aim is to construct a family of neural networks $F \colon \R^d \to \R$ of depth $L+2$ for $L\in \N$ such that $\beta_k(F)$ grows exponentially in the size of the neural network for all $k \in [d-1]_0$. 

We propose a construction that is restricted to architectures where the widths $n_1,\ldots, n_{L+1}$ of all hidden layers but the last one are divisible by $2d$. This construction, however, is generalized for any architecture where the dimension of all hidden layers is at least $2d$ by inserting at most $2d$ auxiliary neurons at each layer at which a zero map is computed. Correspondingly, one obtains a lower bound by rounding down the width $n_\ell$ at each layer to the largest possible multiple of $2d$.   
    In particular, a reduction to the case in Theorem~\ref{theorem:main} in the appendix does not have an effect on the asymptotic size of the NN.

The key idea is to construct $F=f \circ h$ as a consecutive execution of two neural networks $f$ and $h$, where the map $h \colon \R^d \to \R^d$ is an ReLU NN with $L$ hidden layers that identifies exponentially many regions with each other. More precisely, $h$ cuts the unit cube of $\R^d$ into exponentially many small cubes $W_{i_1,\ldots,i_d} \in [0,1]^d$ and maps each of these cubes to the whole unit cube by scaling and mirroring.   
The one hidden layer ReLU NN $f$ then cuts the unit cube into pieces so that $f$ on the pieces alternatingly takes exclusively positive respectively negative values (cf.~Figures~\ref{fig:fig3} and~\ref{fig:fig4}). Since $h$ maps all $W_{i_1,\ldots,i_d}$ to $[0,1]^d$ by scaling and mirroring, every $W_{i_1,\ldots,i_d}$ is cut into positive-valued and negative-valued regions by the composition $f \circ h$ in the same way as $[0,1]^d$ is mapped by $f$, up to mirroring. The cutting of the unit cube and the mirroring of the small cubes in the map to $[0,1]^d$ are chosen in such a way that the subspaces on which $F$ takes negative values assemble at the corners (each corner belongs to $2^d$ small cubes by the cutting by $h$) of the cubes so that they form $k$-annuli for every $k\in [d-1]$. Since $h$ cuts the unit cube into exponentially many small cubes, we obtain exponentially many $k$-annuli for every $k\in [d-1]$ in $F^{-1}((-\infty,0))$~(cf.~Figures\ref{fig:fig5} and~\ref{fig:fig6}). Some technical adjustments will then yield the result for 
$F^{-1}((-\infty,0])$.

The idea of constructing a ReLU neural network that folds the input space goes back to \citet{LinearRegions_Montufar}, where the construction was used to show that a deep neural network with ReLU activation function can have exponentially many linear regions. For our purposes, we explicitly state the continuous piecewise linear map that arises from the construction instead of proving only the existence of such a neural network. 
Using their techniques, we first build a 1-hidden layer NN $h^{(1,m,d)} \colon \R^d \to \R^d$ for $m \in \mathbb{N}$ even that folds the input space, mapping $m^d$ many small cubes $W^{(1,m,d)}_{i_1,\ldots,i_d} \subset [0,1]^d$ by scaling and mirroring to $[0,1]^d$. More precisely, the NN $h^{(1,m,d)}$ has $m \cdot d$ many neurons in the single hidden layer which are partitioned into $m$ groups.
The weights are chosen such that the output of the neurons in one group  depends only on one input variable and divides the interval $[0,1]$ into $m$ subintervals of equal length, each of which is then mapped to the unit interval $[0,1]$ by the output neuron. Figure \ref{fig:construction_folding} illustrates this construction. In Appendix~\ref{foldingnetwork} or in 
\citet{LinearRegions_Montufar}, the reader can find an explicit construction of $h^{(1,m,d)}$.
\begin{figure}
        \begin{minipage}[valign=b]{.45\textwidth}
   	         \begin{center}
		\begin{tikzpicture}[scale=1.8]
				\filldraw [black] (0,1) circle (1pt)node[anchor=east]{$1$};
					\filldraw [black] (1,0) circle (1pt)node[anchor=north]{$1$};
					\filldraw [black] (0.5,0) circle (1pt)node[anchor=north]{$0.5$};
			\draw[connection] (-1,0) -- (2,0);
			\draw[connection] (0,-0.2) -- (0,1.5);
			\draw[line width=2pt] (0,0) -- (0.5,1);
   \draw[line width=2pt] (0.5,1) -- (1,0);

		\end{tikzpicture}\vspace{0.5em}
	\end{center}
 \caption{The graph of the function \mbox{$\pi_j \circ h^{(1,2,d)}$} that folds the unit interval, i.e., mapping the interval $[0,0.5]$ and $[0.5,1]$ to the unit interval. This function is realised by a hidden layer with $2$ hidden neurons.}
     \label{fig:construction_folding}
	\end{minipage}%
   \hspace{.1\linewidth}
    \begin{minipage}[valign=b]{.45\textwidth}
    \centering

	     \begin{tikzpicture}[]
		\footnotesize
		\node[neuron] (x1) at (0,10ex) {$x_1$};
		\node[neuron] (x2) at (0,0) {$x_{d}$};
		\node[neuron] (n11) at (11ex,15ex) {\ReLU};
		\node[neuron] (n12) at (11ex,10ex) {\ReLU};
		\node[neuron] (n13) at (11ex,0ex) {\ReLU};
		\node[neuron] (n15) at (11ex,-5ex) {\ReLU};
		\draw[connection] (x1) -- (n11);
		\draw[connection] (x1) -- (n12);
	    \node (ef) at (11ex,6ex) {$\vdots$};
		\draw[connection] (x2) -- (n13);
		\draw[connection] (x2) -- (n15);
		\node[neuron] (n21) at (22ex,10ex){$y_1$};
		\node[neuron] (n22) at (22ex,0){$y_{d}$};
		\draw[connection] (n11) -- (n21);
		\draw[connection] (n12) -- (n21);
		\draw[connection] (n13) -- (n22);
		\draw[connection] (n15) -- (n22);
	\end{tikzpicture}
	     \caption{The architecture of the one hidden layer neural network $h^{(1,2,d)}$ that folds the $d$-dimensional unit cube by folding every component of the cube as described in Figure 1.}
	     \label{fig:1hiddenlayer_folding}
    \end{minipage}
\end{figure}

The map $h^{(1,m,d)}$ identifies only $O(m^d)$ many cubes with each other. To subdivide the input space into exponentially many cubes and map them to the unit cube, we need a deep neural network. For this purpose, we utilize a vector $\mathbf{m}$ of folding factors instead of a single number $m$. Let \mbox{$\mathbf{m}=(m_1,\ldots,m_L) \in \N^L$} with $m_\ell$ even for all $\ell\in[L]$ and define the neural network $h^{(L,\mathbf{m},d)}$ with $L$ hidden layers as $h^{(L,\mathbf{m},d)} = h^{(1,m_L,d)} \circ \cdots \circ h^{(1,m_1,d)}$. Since each of the $m_1^d$ cubes that results from the subdivision by the first layer is mapped back to $[0,1]^d$, each cube is subdivided again into $m_2^d$ cubes by the subsequent layer. Thus, after $L$ such layers, we obtain a subdivision of the input space into $\left(\prod_{\ell=1}^Lm_\ell\right)^d$ cubes.

In the following, we define fixed but arbitrary variables: $L \in \mathbb{N}$, \mbox{$\mathbf{m}=(m_1,\ldots,m_L) \in \N^L$} and $M = \left(\prod_{\ell=1}^Lm_\ell\right)$ with $m_\ell >1$ even for all $\ell \in [L]$. The following lemma recollects the existence of a map $h^{(L,\mathbf{m},d)}$ that actually enjoys the aforementioned properties.

\begin{restatable}{lemma}{babycubes} (cf. \citep{LinearRegions_Montufar})
\label{lemma:babycubes}
    Let $d \in \mathbb{N}$. Then there exists a map $h^{(L,\mathbf{m},d)}\colon \R^d\to\R^d$ such that for all $(i_1,\ldots,i_{d}) \in [M]^{d}$, the following hold: \begin{enumerate}
	\item $h^{(L,\mathbf{m},d)}(W^{(L,\mathbf{m},d)}_{(i_1,\ldots,i_{d})})= [0,1]^{d}$
    \item $\pi_j \circ h^{(L,\mathbf{m},d)} _{|W^{(L,\mathbf{m},d)}_{(i_1,\ldots,i_{d})}}(x_1,\ldots,x_{d})= \left\{
	\begin{array}{ll}
		M\cdot x_j -(i_j-1) & i_j \text{ odd} \\
		-M\cdot x_j + i_j  & \, i_j\text{ even} \\
	\end{array} \right.$ \end{enumerate} 
\end{restatable}

We now define cutting points as the points that are mapped to the point $(1,1,1,\ldots.,1,0)$ by the map $h^{(L,\mathbf{m},d)}$. They will play a central role in counting the annuli in the sublevel set of $F$. 
\begin{def1}
  We call a point $x\in [0,1]^{d}$ a \emph{cutting point} if it has coordinates of the form $x_i=\frac{x'_i}{M}$ for all $i\in \{1,\ldots, d\}$, where the $x'_i$ are odd integers for $1\leq i \leq d-1$ and $x'_d$ is an even integer.
\end{def1} 

\par Next, for $w \geq 2$, we build a 1-hidden layer neural network $\hat{g}^{(w,d)} \colon \R^d \to \R$ that cuts the d-dimensional unit cube into $w$ pieces such that $\hat{g}^{(w,d)}$ maps the pieces alternatingly to $\R_{\geq0}$ and $\R_{\leq0}$, respectively. We omit the indices $w$ and $d$ whenever they are clear from the context. 

In order to build the neural network, we fix $w$ and $d$ and define the maps $\hat{g}_q \colon \R^{d} \to \R$, \mbox{$\ q =0,\ldots,w+1$} by \[\hat{g}_q(x) =   \left\{
\begin{array}{ll}
	\max\{0,\1x \} & q=0 \\
	\max\{0,\1x-\frac{1}{4}\} & q=w+1 \\
	\max\{0,2(\1x-(2q-1)/8w)\} & \, \textrm{else} \\
\end{array} \right.\]

\begin{figure}

\begin{minipage}[t]{.4\linewidth}
\centering
\begin{tikzpicture}[scale=0.4] 
\draw[step = 2.0,gray,very thin]
(0,0) grid (8,8);
\fill[darkgray] (0,0) -- (0,2) -- (2,0) -- cycle;
\fill[lightgray] (0,2) -- (0,4) -- (4,0) -- (2,0) -- cycle;
\fill[darkgray] (0,4) -- (0,6) -- (6,0) -- (4,0) -- cycle;
\fill[lightgray] (0,6) -- (0,8) -- (8,0) -- (6,0) -- cycle;

\draw[name path=y-axis,->] (0,0) -- (0,9) node[above]{$x_2$};
\draw[->, name path=x-axis] (0,0) -- (9,0) node[right]{$x_1$};

\draw (8,0) -- (8,8);
\draw (8,8) -- (0,8);

\draw[black] (0,2) -- (2,0);
\draw[black] (0,4) -- (4,0);
\draw[black] (0,6) -- (6,0);
\draw[black] (0,8) -- (8,0);
\draw [dashed](2,-2) -- (-2,2) node[left] {$g_0$};
\draw[dashed] (3,-2) -- (-2,3)  node[left] {$g_1$};
\draw[dashed] (5,-2) -- (-2,5) node[left] {$g_2$};
\draw[dashed] (7,-2) -- (-2,7)  node[left] {$g_3$};
\draw[dashed] (9,-2) -- (-2,9) node[left] {$g_4$};
\draw[dashed] (10,-2) -- (-2,10) node[left] {$g_5$};
\filldraw [black] (0,8) circle (2pt)node[anchor=east]{$\frac{1}{4}$};
\filldraw [black] (8,0) circle (2pt)node[anchor=north]{$\frac{1}{4}$};
\end{tikzpicture}
\caption{Illustration of the preimage of the map $\hat{g}^{(4,2)}$ in $[0,\frac{1}{4}]^2$, where the dark gray regions correspond to $\hat{g}^{-1}((0,\infty))$ and the light gray regions to $\hat{g}^{-1}((-\infty,0))$.}
\label{fig:fig3}
\end{minipage}
   \hspace{.1\linewidth}%
\begin{minipage}[t]{.4\linewidth}
\centering
\begin{tikzpicture}[scale=0.4] 
\draw[step = 2.0,gray,very thin]
(0,0) grid (8,8);
\fill[darkgray] (8,0) -- (8,0.5) -- (7.5,0) -- cycle;
\fill[lightgray] (8,0.5) -- (8,2) -- (7,0) -- (7.5,0) -- cycle;
\fill[darkgray] (8,1) -- (8,1.5) -- (6.5,0) -- (7,0) -- cycle;
\fill[lightgray] (8,1.5) -- (8,2) -- (6,0) -- (6.5,0) -- cycle;

\draw[name path=y-axis,->] (0,0) -- (0,9) node[above]{$x_2$};
\draw[->, name path=x-axis] (0,0) -- (9,0) node[right]{$x_1$};

\draw (8,0) -- (8,8);
\draw (8,8) -- (0,8);

\draw[black] (8,0.5) -- (7.5,0);
\draw[black] (8,1) -- (7,0);
\draw[black] (8,1.5) -- (6.5,0);
\draw[black] (8,2) -- (6,0);

\filldraw [black] (0,8) circle (2pt)node[anchor=east]{$1$};
\filldraw [black] (8,0) circle (2pt)node[anchor=north]{$1$};
\end{tikzpicture}
\caption{Illustration of the preimage of the map $g^{(4,2)}$ in $[0,1]^2$.}
\label{fig:fig4}
\end{minipage}
\end{figure}

Later in this section, we will iteratively construct $k$-annuli in the sublevel set of $F$ for all \mbox{$k \in [d-1]$}. In order to ensure that these annuli are disjoint, it is convenient to place them around the cutting points. To achieve this, we mirror the map $\hat{g}$ before precomposing it with $h$. The mirroring transformation that maps the origin to the point $(1,\ldots, 1,0)$ is an affine map $t:[0,1]^d\rightarrow[0,1]^d$ defined by $t(x_1,x_2,\ldots,x_d)=(1-x_1,1-x_2,\ldots,1-x_{d-1},x_d)$. We define the neural network $g = \hat{g} \circ t$ as the consecutive execution of $\hat{g}$ and $t$.
\begin{restatable}{lemma}{mirroredcuts}
\label{lemma:mirrored_cuts}
    Let $d,w \in \N$ with $w$ odd and  \[R_q =\{x \in [0,1]^{d}: \frac{q}{4w} < \|(1,1,\ldots,1,0)-x\|_1 < \frac{q+1}{4w}\}.\]
     Then there exists a 1-hidden layer neural network $g^{(w,d)} \colon \R^d \to \R$ of width $w+2$ such that $g^{(w,d)}(R_q) \subseteq (-\infty,0)$ for all odd $  \in [w-1]_0, g^{(w,d)}(R_q) \subseteq  (0,\infty)$ for all even $q \in [w-1]_0$ and $g^{(w,d)}(x)=0$ for all $x \in [0,1]^{d}$ with $\|(1,1,\ldots,1,0)-x\|_1\geq\frac{1}{4}.$ 
\end{restatable}

Lemma~\ref{lemma:onionrings} in the appendix characterizes the regions around cutting points that admit positive respectively negative values under the map $g^{(w,d)} \circ h^{(L,\mathbf{m},d)}$. We focus on the regions that admit negative values, i.e., the space $Y_{d,w}\coloneqq (g^{(w,d)} \circ h^{(L,\mathbf{m},d)})^{-1}((-\infty,0))$ and observe that we obtain $d$-dimensional \mbox{$(d-1)$-annuli} around each cutting point.
\begin{figure}
   \begin{minipage}[t]{.4\linewidth}
   \centering
   \begin{tikzpicture}[scale = 0.55]
    \draw[gray,very thin] (0,0) grid (8,8);
    \draw[name path=y-axis,->] (0,0) -- (0,9) node[above]{$x_2$};
    \draw[->, name path=x-axis] (0,0) -- (9,0) node[right]{$x_1$};
    \draw (8,0) -- (8,8);
    \draw (8,8) -- (0,8);
    \foreach \i in {1,3,5,7}{
        \foreach \j in{2,4,6}{
            \fill[lightgray] (\i-4/8,\j) -- (\i,\j-4/8) -- (\i+4/8,\j) -- (\i,\j+4/8) -- cycle;
            \fill[darkgray] (\i-3/8,\j) -- (\i,\j-3/8) -- (\i+3/8,\j) -- (\i,\j+3/8) -- cycle;
            \fill[lightgray] (\i-2/8,\j) -- (\i,\j-2/8) -- (\i+2/8,\j) -- (\i,\j+2/8) -- cycle;
            \fill[darkgray] (\i-1/8,\j) -- (\i,\j-1/8) -- (\i+1/8,\j) -- (\i,\j+1/8) -- cycle;   
        }
    }
    \foreach \i in {1,3,5,7}{
            \fill[lightgray] (\i-4/8,8) -- (\i,8-4/8) -- (\i+4/8,8) -- cycle;
            \fill[darkgray] (\i-3/8,8) -- (\i,8-3/8) -- (\i+3/8,8) --  cycle;
            \fill[lightgray] (\i-2/8,8) -- (\i,8-2/8) -- (\i+2/8,8) -- cycle;
           \fill[darkgray] (\i-1/8,8) -- (\i,8-1/8) -- (\i+1/8,8) -- cycle;  
    }
    \foreach \i in {1,3,5,7}{
            \fill[lightgray] (\i-4/8,0) -- (\i,4/8) -- (\i+4/8,0) -- cycle;
            \fill[darkgray] (\i-3/8,0) -- (\i,3/8) -- (\i+3/8,0) --  cycle;
            \fill[lightgray] (\i-2/8,0) -- (\i,2/8) -- (\i+2/8,0) -- cycle;
            \fill[darkgray] (\i-1/8,0) -- (\i,1/8) -- (\i+1/8,0) -- cycle;   
    }
    \filldraw [black] (0,8) circle (2pt)node[anchor=east]{$1$};
\filldraw [black] (8,0) circle (2pt)node[anchor=north]{$1$};
   \end{tikzpicture}    
    \caption{Illustration of the preimage of the composition \mbox{$g^{(4,2)} \circ h^{(3,2,2)}$}.}
    \label{fig:fig5}
   \end{minipage}
   \hspace{.1\linewidth}%
    \begin{minipage}[t]{.45\linewidth}
    \centering
    \begin{tikzpicture}[scale = 0.55]
    \fill[lightgray] (8-2,8,8) -- (8,8-2,8) -- (8,8,8-2) -- cycle;
    \draw (8-2,8,8) -- (8,8-2,8) -- (8,8,8-2) -- cycle;
    \fill[darkgray] (8-1.5,8,8) -- (8,8-1.5,8) -- (8,8,8-1.5) -- cycle;
    \draw (8-1.5,8,8) -- (8,8-1.5,8) -- (8,8,8-1.5) -- cycle;
    \fill[lightgray] (8-1,8,8) -- (8,8-1,8) -- (8,8,8-1) -- cycle;
    \draw (8-1,8,8) -- (8,8-1,8) -- (8,8,8-1) -- cycle;
    \fill[darkgray] (8-0.5,8,8) -- (8,8-0.5,8) -- (8,8,8-0.5) -- cycle;
    \draw (8-0.5,8,8) -- (8,8-0.5,8) -- (8,8,8-0.5) -- cycle;

    \fill[lightgray] (8,2,8) -- (8-2,0,8) -- (8,0,8) -- (8,0,0) -- (8,2,0) -- cycle;
   \draw (8,2,8) -- (8-2,0,8)-- (8,2,0) -- cycle;
    \fill[darkgray] (8,1.5,8) -- (8-1.5,0,8) -- (8,0,8) -- (8,0,0) -- (8,1.5,0) -- cycle;
    \draw (8-1.5,0,8) -- (8,1.5,8) -- (8,1.5,0);
    \fill[lightgray] (8,1,8) -- (8-1,0,8) -- (8,0,8) -- (8,0,0) -- (8,1,0) -- cycle;
     \draw (8-1,0,8) -- (8,1,8) -- (8,1,0);
    \fill[darkgray] (8,0.5,8) -- (8-0.5,0,8) -- (8,0,8) -- (8,0,0) -- (8,0.5,0) -- cycle;
    \draw (8-0.5,0,8) -- (8,0.5,8) -- (8,0.5,0);;
      
    \foreach \x in{2,4,6}
{   \draw[gray,very thin] (0,\x ,8) -- (8,\x ,8);
    \draw[gray,very thin] (\x ,0,8) -- (\x ,8,8);
    \draw[gray,very thin] (8,\x ,8) -- (8,\x ,0);
    \draw[gray,very thin] (\x ,8,8) -- (\x ,8,0);
    \draw[gray,very thin] (8,0,\x ) -- (8,8,\x );
    \draw[gray,very thin] (0,8,\x ) -- (8,8,\x );
}
    \foreach \x in{0,8}
{   \draw (0,\x ,8) -- (8,\x ,8);
    \draw (\x ,0,8) -- (\x ,8,8);
    \draw (8,\x ,8) -- (8,\x ,0);
    \draw (\x ,8,8) -- (\x ,8,0);
    \draw (8,0,\x ) -- (8,8,\x );
    \draw (0,8,\x ) -- (8,8,\x );
}

\draw[name path=y-axis,->] (0,0,8) -- (0,9,8) node[above]{$x_2$};
\draw[->, name path=x-axis] (0,0,8) -- (9,0,8) node[right]{$x_1$};
\draw[->, name path=z-axis] (8,0,8) -- (8,0,-2) node[right]{$x_3$};

\draw[dashed] (0,0,8) -- (0,0,0);
\draw[dashed] (0,0,0) -- (0,8,0);\
\draw[dashed] (0,0,0) -- (8,0,0);
\filldraw [black] (0,8,8) circle (2pt)node[anchor=east]{$1$};
\filldraw [black] (8,0,8) circle (2pt)node[anchor=north]{$1$};
\filldraw [black] (0,0,0) circle (2pt)node[anchor=north]{$1$};
\end{tikzpicture}
    \caption{Illustration of the preimage of $f^{(4,4)} = g^{(4,3)} + g^{(4,2)} 
    \circ p_2$.}
    \label{fig:fig6}
   \end{minipage}
\end{figure}

In order to count the annuli we count the cutting points.
\begin{obs}\label{obs:evenpts} 
  Cutting points lie on a grid in the unit cube, with $\frac{M}{2}$ many cutting points into dimensions $1, \ldots, d-1$ and $\frac{M}{2}+1$ many in dimension $d$. Thus, there are $\frac{M^{(d-1)}}{2^{d-1}}\cdot\left(\frac{M}{2}+1\right)$ cutting points. Note that since $M$ is an even number, these points cannot lie on the boundary unless the last coordinate is $0$ or $M$. This means, $2\cdot \frac{M^{(d-1)}}{2^{d-1}}=\frac{M^{(d-1)}}{2^{d-2}}$ of the cutting points are located on the boundary of the unit cube and the remaining  $\frac{M^{(d-1)}}{2^{d-1}}\cdot\left(\frac{M}{2}-1\right)$ are in the interior.
\end{obs}

Combining Lemma \ref{lemma:onionrings} and  Observation~\ref{obs:evenpts}, we can finally describe $Y_{d,w}$ as a topological space.

\begin{restatable}{prop}{homeoone}  
\label{prop:homeo1}   The space $Y_{d,w}$ is homeomorphic to the disjoint union of \mbox{$p_d=\frac{M^{(d-1)}}{2^{d-1}}\cdot\left(\frac{M}{2}-1\right) \cdot\left\lceil\frac{w}{2}\right\rceil$} many $(d-1)$-annuli and $p'_d = \frac{M^{(d-1)}}{2^{d-2}}\cdot\left\lceil \frac{w}{2}\right\rceil$ many disks, that is, \[Y_{d,w}\cong\bigsqcup\limits_{k=1}^{p_d}(S^{d-1}\times [0,1])\sqcup \bigsqcup\limits_{k=1}^{p'_d}D^d.\]
\end{restatable}

In order to obtain exponentially many $k$-annuli for all $k \in [d-1]$, we follow a recursive approach: At each step, we start with a $k$-dimensional space that has exponentially many $j$-annuli for all $j \in[k-1]$. We then cross this space with the interval $[0,1]$, transforming the $k$-dimensional $j$-annuli into $(k+1)$-dimensional $j$-annuli. Finally, we ``carve'' $(k+1)$-dimensional $k$-annuli in this newly formed product space. To allow us flexibility with respect to the numbers of annuli carved in different dimensions, we fix an arbitrary  vector
$\textbf{w}=(w_1,\ldots,w_{d-1}) \in \N^{d-1}$ such that \mbox{$\sum_{i=1}^{d-1}(w_i+2) = n_{L+1}$}. We iteratively define the 1-hidden layer neural network \[f^{(w_1,...,w_{k-1})} \colon \R^k \to \R\]
of width $n_{L+1}$ by $f^{(w_1)} = g^{(w_1,2)}$ and 
\[f^{(w_1,\ldots,w_{k-1})} = f^{(w_1,\ldots,w_{k-2})} \circ p_{k-1} + g^{(w_{k-1},k)}\] 
for $k \leq d$. Roughly speaking, the following lemma states that the carving map does not interfere with the other maps, i.e., there is enough space in the unit cubes to place the $k$-annuli after having placed all $k'$-annuli ($k'< k$) in the same, inductive manner.

\begin{restatable}{lemma}{enoughspace}
\label{lemma:enoughspace}
For $k \leq d$ and $\mathbf{w} = (w_1,\ldots,w_{d-1})\in \N^{d-1}$ it holds that \begin{enumerate}
    \item  $f^{(w_1,...,w_{k-2})} \circ p_{k-1}(x) \neq 0  \implies g^{(w_{k-1},k)}(x) = 0$ and 
    \item  $g^{(w_{k-1},k)}(x) \neq 0 \implies f^{(w_1,...,w_{k-2})} \circ p_{k-1}(x) = 0 $
\end{enumerate}
for all $x \in [0,1]^k$.
\end{restatable}

Using Lemma \ref{lemma:enoughspace} and the fact that the folding maps $h^{(L,\mathbf{m},k)}$ are compatible with projections (cf.~Lemma~\ref{lemma:commute} in Appendix), we can make sure that we can construct the cuts iteratively so that we obtain $k$-annuli for every $k \in [d-1]$, which is stated in the following lemma.

\begin{restatable}{lemma}{homeo}
\label{lemma:homeo}
For $2 \leq k \leq d$, the space $X_k\coloneqq (f^{(w_1,...,w_{k-1})}\circ h^{(L,\mathbf{m},k)})^{-1}((-\infty,0))$ satisfies
    \[X_k=\left(X_{k-1}\times [0,1]\right)\sqcup Y_{k,w_{k-1}}\] with $X_1 \coloneqq \emptyset$.
\end{restatable}

Lemma~\ref{lemma:homeo}, Proposition~\ref{prop:homeo1} and the disjoint union axiom of singular homology (Proposition~\ref{topoez} in Appendix~\ref{top_background}) allow us to compute the Betti numbers of the decision region of $F \coloneqq f^{(w_1,\ldots,w_{d-1})}\circ h^{(L,\mathbf{m},d)}$ as stated in Theorem~\ref{theorem:main} in the appendix. One can easily generalize this statement by rounding down the widths $n_1,\ldots, n_L$ to the nearest even multiple of $d$:
\begin{restatable}{theorem}{main}
\label{theorem:main}
      Given an architecture $A=(d,n_1,\ldots,n_L,1)$ with $n_\ell \geq 2d$ for all $\ell\in [L]$ and numbers $w_1,\ldots, w_{d-1} \in \N$ such that $\sum_{k=1}^{d-1}(w_k +2) = n_L$, there is a neural network $F \in \mathcal{F}_A$ with weights bounded from above by $\max_{\ell = 1, \ldots L}2\frac{n_\ell}{d}$ such that     
          \begin{enumerate}[label=(\roman*)]
  \item $\beta_0(F^{-1}((-\infty,0))) = \sum_{k=2}^d\frac{M^{(k-1)}}{2^{k-1}}\cdot\left(\frac{M}{2}+1\right)\cdot\left\lceil \frac{w_k}{2}\right\rceil$
     \item $\beta_k(F^{-1}((-\infty,0)))=\frac{M^{(k-1)}}{2^{k-1}}\cdot\left(\frac{M}{2}-1\right) \cdot\left\lceil\frac{w_{k-1}}{2}\right\rceil$ \  for $0<k < d$,
    \end{enumerate}
    where $M = {\prod_{\ell=1}^{L-1} 2 \cdot \lfloor \frac{n_\ell}{2d}\rfloor}$.  
\end{restatable}

 In order to obtain lower bounds for $\beta_k(F)$, we modify the construction slightly by adding a small constant $b$ to the output layer, which yields a neural network $F'$ such that there is no full-dimensional region $R$ such that $F'(R) = \{0\}$. The construction then yields that $\overline{F'^{-1}((-\infty,0))}= F'^{-1}((-\infty,0])$. Since adding a small constant $b$ only makes the annuli in the sublevel set $F^{-1}((-\infty,0))$ thinner, the spaces
$F^{-1}((-\infty,0))$ and $F'^{-1}((-\infty,0))$ are homeomorphic. Furthermore, since annuli are homotopy equivalent to their closures, the sublevel set
  $F'^{-1}((-\infty,0))$ is homotopy equivalent to its closure
\mbox{$\overline{F'^{-1}((-\infty,0))}= F'^{-1}((-\infty,0])$}. Since Betti numbers are invariant under homotopy equivalences, it follows that $\beta_k(F') = \beta_k(F^{-1}((-\infty, 0))$ for all $k \in [d-1]_0$, resulting in the following theorem.
 
\begin{restatable}{theorem}{closure}
\label{cor:exactformulaclosure}
        Given an architecture $A=(d,n_1,\ldots,n_L,1)$ with $n_\ell \geq 2d$ for all $\ell\in [L]$ and numbers $w_1,\ldots, w_{d-1} \in \N$ such that $\sum_{k=1}^{d-1}(w_k +2) = n_L$, there is a neural network $F \in \mathcal{F}_A$ with weights bounded from above by $\max_{\ell = 1, \ldots L}2\frac{n_\ell}{d}$ such that 
         
          \begin{enumerate}[label=(\roman*)]
  \item $\beta_0(F) = \sum_{k=2}^d\frac{M^{(k-1)}}{2^{k-1}}\cdot\left(\frac{M}{2}+1\right)\cdot\left\lceil \frac{w_k}{2}\right\rceil$
     \item $\beta_k(F)=\frac{M^{(k-1)}}{2^{k-1}}\cdot\left(\frac{M}{2}-1\right) \cdot\left\lceil\frac{w_{k-1}}{2}\right\rceil$ \  for $0<k < d$,
    \end{enumerate}
where $M = {\prod_{\ell=1}^{L-1} 2 \cdot \lfloor \frac{n_\ell}{2d}\rfloor}$.  
\end{restatable}
The special case $\left\lfloor\frac{w_1}{2}\right\rfloor=\ldots=\left\lfloor\frac{w_d}{2}\right\rfloor$ then corresponds precisely to Corollary~\ref{cor:main}.

As mentioned previously, the sum of Betti numbers, the notion of topological expressivity used in \citet{Bianchini2014OnTC}, does not provide us with an understanding of holes of different dimensions. On the other hand, our bounds are an extension of this result. In addition, the dimension-wise lower bound allows further implications, one of them being a lower bound on the \emph{Euler characteristic}, which is the alternating sum $\chi(X)=\sum_{k=1}^d\beta_k(X)$ of the Betti numbers. 
\begin{restatable}{cor}{eulerchar}
\label{cor:eulerchar}
    Let $A$ be the architecture as in Theorem~\ref{cor:exactformulaclosure} , then there is a ReLU NN $F \colon \R^d \mapsto \R$ with architecture $A$ such that $\chi(F^{-1}((-\infty,0]))\in \Omega\left(M^d\cdot \sum\limits_{i=1}^{d-1}w_i\right)$, where $\chi(F^{-1}((-\infty,0]))$ denotes the Euler characteristic of the space $F^{-1}((-\infty,0])$. 
\end{restatable}
\begin{proof}
    The Euler characteristic of a finite CW complex $X$ is given by the alternating sum of its Betti numbers, i.e., by the sum $\sum_{k\in \N}(-1)^{k}\beta_k(X)$. By Theorem~\ref{theorem:main}, this term is dominated by the zeroth Betti number, from which the claim follows.
\end{proof}
The Euler characteristic is an invariant used widely in differential geometry in addition to algebraic topology. For instance, it can also be defined by means of the index of a vector field on a compact smooth manifold.

\section{Topologically Stable ReLU Neural Networks}

In the following, we establish a sufficient criterion for the parameters of a neural network such that the topological expressivity of the corresponding neural networks is constant in an open neighbourhood of this parameter.  
The neural network constructed explicitly in Section~\ref{lowerbound} to obtain the lower bound fulfills this criterion, so that the lower bound is attained in an open subset of the parameter space.

We denote by $\Phi \colon \R^D \to C(\R^d)$ the map that assigns a vector of weights in the parameter space \mbox{$\R^D \cong \bigoplus^{L+1}_{\ell=1} \R^{(n_{\ell-1} +1) \times n_{\ell}}$} to the function computed by the ReLU neural network with this weights, i.e.,
\[\Phi(p) \coloneqq T_{L+1}(p) \circ \sigma_{n_L} \circ T_{L}(p) \circ \cdots \circ \sigma_{n_1} \circ T_1(p)\](c.f. Definition~\ref{def:realization map} in the appendix) and by  \[\Phi^{(i,\ell)}(p) \coloneqq \pi_i \circ T_{\ell}(p) \circ \cdots \circ \sigma_{n_1} \circ T_0(p),\] the map computed at the $i$-th neuron in layer $\ell$. Any neuron $(i,\ell)$ defines a hyperplane \[H_{i,\ell}(p) = \{x \in \R^{n_{\ell-1}} \mid \pi_i \circ \left(T_\ell(p)\right)(x) = 0\}\] in the output space of the previous layer. 
One can iteratively define a sequence of polyhedral complexes $\mathcal{P}^{(i,\ell)}(p)$ such that $\Phi^{(k,j)}(p)$ is affine linear on the polyhedra of $\mathcal{P}^{(i-1,\ell)}(p)$ for all $(k,j)$ lexigrophically smaller than $(i,\ell)$ by intersecting all the polyhedra in $\mathcal{P}^{(i,\ell)}(p)$ with the pullback of the hyperplane $H_{i,\ell}(p)$ to the input space i.e., $\{x \in \R^d \mid \left(\Phi^{(i,\ell)}(p)\right)(x)=0\}$ and the pullbacks of the corresponding half-spaces to the input space i.e.,  $\{x \in \R^d \mid \left(\Phi^{(i,\ell)}(p)\right)(x)\leq0\}$ and $\{x \in \R^d \mid \left(\Phi^{(i,\ell)}(p)\right)(x)\geq 0\}$ (c.f. Definition~\ref{def:canonicalPC} in the appendix). For a polytope $K$, let $\hat{K} \coloneqq \{F \mid F$ is a face of $K\}$ be the polyhedral complex consisting of the faces of $K$. 
The canonical polyhedral complex (with respect to $K$) is then defined as $\mathcal{P}(p,K) \coloneqq \mathcal{P}^{(n_L,L)}(p) \cap \hat{K}$ (c.f. \cite{benthyperplane}). Furthermore, $\Phi(p)$ is affine linear and non-positive respectively non-negative on all polyhedra of $\mathcal{P}^{(n_{L+1},1)}(p) \cap \hat{K}$, which is a refinement of the canonical polyhedral complex.

We call a neural network $\Phi(p)$ \emph{topologically stable with respect to $K$} if the pullback of the hyperplane $H_{i,\ell}(p)$ does not intersect any vertices (i.e., faces of dimension $0$) of $\mathcal{P}^{(i,\ell)}(p) \cap \hat{K}$ for all neurons $(i,\ell)$ (c.f. Definition~\ref{def:topstable} in the appendix). One can perturb the weights (and hence the hyperplanes as well as their pullbacks) of a topologically stable neural network within a small enough magnitude such that the combinatorial structure of the refinement $\mathcal{P}^{(n_{L+1},L+1)}(p) \cap \hat{K}$ is preserved.
 \begin{prop}(c.f. \Cref{prop:combstable} in the appendix).
    Let $K$ be a polytope and $\Phi(p) \colon K \to \R$ be a topologically stable (w.r.t to $K$) ReLU neural network of architecture $(n_0, \ldots n_{L+1})$. Then for every $\delta > 0$, there is an open set $U \subseteq \R^D$ such that for every $u \in U$ there is an isomorphism $\varphi_u \colon \mathcal{P}(p,K) \to \mathcal{P}(u,K)$ with $\|v - \varphi_u(v)\|_2 < \delta$ for every vertex $v \in \mathcal{P}(p,K)$.
    \end{prop}
    Since the isomorphisms extend to the refinements $\mathcal{P}^{(n_{L+1},1)}(p) \cap \hat{K}$ and $\mathcal{P}^{(n_{L+1},1)}(u) \cap \hat{K}$ and an isomorphism of polyhedral complexes $\varphi \colon \mathcal{P} \to \mathcal{Q}$ yield a PL-homeomorphism between the respective supports $|\varphi| \colon |\mathcal{P}| \to  |\mathcal{Q}|$, we obtain the following result. 
\begin{restatable}{prop}{topstable}
    \label{prop:topstable}
        Let $K$ be a polytope and $\Phi(p)$ a topologically stable ReLU neural network with respect to $K$, then there is a \mbox{$\delta >0$} such that for all $u \in B_\delta(p)$ it holds that $K \cap \Phi(p)^{-1}((-\infty,0])$ is homeomorphic to $K \cap \Phi(u)^{-1}((-\infty,0])$.
    \end{restatable}

Lastly, since our construction for the lower bound yield topologically stable ReLU NN with respect to the unit cube, Proposition~\ref{prop:topstable} implies that our results are stable with respect to small perturbations.

\begin{restatable}{prop}{robust}
\label{prop:robust_to_pertubation}
    There is an open set $U \subseteq \R^D$ in the parameter space of the architecture {$A=(d,n_1,\ldots,n_L,1)$} such that $\Phi(u)$ restricted to the unit cube has at least the same topological expressivity as $F$ in Theorem~\ref{cor:exactformulaclosure} for all $u \in U.$ 
\end{restatable}

\section{Upper Bound}
\label{Upperbound}
In this section we derive an upper bound for $\beta_k(F)$ for a ReLU neural network $F\colon \R^d \to \R$ for all $k \in [d-1]$, showing that they are polynomially bounded in the width using an upper bound on the linear regions of $F$. A linear region $R$ of $F$ contains at most one maximal convex polyhedral subspace where $F$ takes on exclusively non-negative function values. Intuitively, every such polyhedral subspace can be in the interior of at most one $d$-dimensional hole of the sublevel set $F^{-1}((- \infty,0])$ and thus the number of linear regions is an upper bound for $\beta_{d-1}(F)$. In the following proposition we will formalize this intuition and generalize it to $\beta_{k}(F)$ for all $k \in [d-1]_0$.

\begin{restatable}{prop}{upperbound}
\label{prop:upperbound}
     Let $F \colon \R^d \to \R$ be a neural network of architecture $(d,n_1,\ldots,n_L,1)$. Then it holds that $\beta_0(F) \leq \sum_{(j_1,\ldots,j_{L})\in J} \quad\prod_{\ell=1}^{L} \binom{n_\ell}{j_\ell}$ and for all $k \in [d-1]$ that
     \[\beta_k(F) \leq \binom{\sum_{(j_1,\ldots,j_{L})\in J} \quad\prod_{\ell=1}^{L} \binom{n_\ell}{j_\ell}}{d-k-s},
     \]  where
    \mbox{$J =\left \{(j_1,\ldots,j_L) \in \mathbb{Z}^L \colon 0 \leq j_\ell \leq \min\{d,n_1-j_1,\ldots,n_{\ell-1}-j_{\ell-1}\} \text{ for all } \ell =1,\ldots,L\right\}$} and $s \in [d]$ is the dimension of the lineality space of a refinement of the canonical polyhedral complex of $F$.
\end{restatable}

\renewcommand*{\proofname}{Proof Sketch}
\begin{proof}
    Theorem 1 in \citep{Bounding_Serra} states that $F$ has at most \mbox{$r \coloneqq \sum_{(j_1,\ldots,j_{L})\in J} \prod_{l=1}^{L} \binom{n_l}{j_l}$} linear regions. In Section~\ref{upperbound} we will provide a formal proof for the statement that we sketch here.
    Let $\Po$ be the canonical polyhedral complex of $F$, i.e, $F$ is affine linear on all polyhedra in $\Po$ (c.f Definition~\ref{def:canonicalPC} in the appendix) and $\Po^-$ be a subcomplex of a refinement of $\Po$ such that $F$ takes on exclusively non-positive values on all polyhedra in $\Po^-$. Therefore, the support $|\Po^-|$ of $\Po^-$ equals $F^{-1}((-\infty,0])$ and we then proceed by showing the chain of inequalities \[\beta_k(F) = \beta_k(|\Po^-|) \leq \#\Po_{k+1-\ell} \leq \binom{r}{d-k-\ell}\] using cellular homology and polyhedral geometry, where $\Po_{k+1}\subseteq \Po$ is the set of $(k+1)$-dimensional polyhedra in $\Po$. This concludes the proof, since it also holds that $\beta_0(|\Po^-|) \leq \#\Po_d = r$.

\end{proof}
\renewcommand*{\proofname}{Proof}

This implies that we obtain an upper bound that is polynomial in the width:
\corupperbound*
\section{Extensions of the Results}
So far, we have demonstrated an exponential gap for the Betti numbers of decision regions computable by deep and shallow networks, where our lower bound was derived using a special construction. However, this construction can be adapted to extend our results to a broader class of classification tasks and neural network architectures.
\paragraph*{Multi-categorical Classification}
      The 
      bounds for binary classification can be achieved for multi-categorical classification as well: Assume that we have $k$ classes that are classified by the intervals $(-\frac{p}{k},-\frac{p-1}{k}]$ for $p \in [k]$. In our constructions, this corresponds to subdividing every annulus into $2k$ annuli (two for each class) and therefore it holds that $2\beta_j(F^{-1}((-\infty,0]) = \beta_j(F^{-1}((-\frac{p}{k},-\frac{p-1}{k}])$ for all $p \in [k]$. To have decision intervals of the form $(0,1],(1,2],\ldots, (k-1,k]$, one can scale this accordingly. Hence, if we have $k$ classes, the Betti numbers of the decision regions of all these classes can be exponentially high simultaneously for deep neural networks. The upper bounds can be achieved in the same way as for the binary case.
\paragraph{Recurrent Neural Networks}
      For $\mathbf{m} = (m, \ldots, m)$, our construction 
      can be seen as a recurrent neural network (RNN) followed by a dense layer: Since we concatenate the same layer map (the folding map $h^{(1,m,d)}$) $L$ times, this can be seen as an RNN and the last layer map (the cutting map $f$) as the dense layer. 
      
      If we restrict ourselves to consider RNNs without an additional dense layer, one can define the two-hidden-layer NN $F^{(w_1,\ldots, w_{d-1},m)} = (\lambda \cdot f^{(w_1,\ldots, w_{d-1})}) \circ h^{(1,m,d)}$, where $\lambda = d \cdot \max_{x \in [0,1]^d} f(x)$. Then, for every cutting point $c$, there exist $2^d$ many cubes who share the vertex $c$ and are mapped back to the unit cube under the map $F^{(w_1,\ldots, w_{d-1},m)}$.
      Therefore, after every subsequent application of $F^{(w_1,\ldots, w_{d-1},m)}$, this behaviour is replicated inside every such cube. Furthermore, there are \mbox{$\Omega(m^d\cdot w_{k-1})$} many $k$-annuli inside every such cube, resulting in $\Omega(m^{dL}\cdot w_{k-1})$ many $k$-annuli after $L$ applications of $F^{(w_1,\ldots, w_{d-1},m)}$.
      Since the upper bounds apply to RNNs regarded as unfolded DNNs as well, we obtain the same exponential gap.

\section{Conclusion, Limitations and Outlook}
Since it is widely accepted that data sets often have nontrivial topologies, investigating a neural network's ability to capture topological properties, as characterized by all Betti numbers, is an exciting and essential question that yields insight into the nature of ReLU networks. In an attempt to shed light on this question, we have proven lower and upper bounds for the topological expressivity of ReLU neural networks with a given architecture. 
Our bounds give a rough estimate on how the architecture needs to be in order to be at least theoretically able to capture the topological complexity of the data set.

Even though our lower bounds apply under certain restrictions of neural network architecture, this does not pose a big limitation for our purposes. Since our results are of a theoretical and mostly asymptotic nature, a constant factor (in the hidden layers resp. in the last hidden layer) is negligible. Besides, since our layers merely consists of many small layers put in parallel, one could also concatenate the layers in order to achieve a smaller width maintaining all the asymptotic results.

As a byproduct of our analysis we have seen that two hidden layers are sufficient to increase the topological expressivity as much as we want at the expense of an increased width. 
Although there are finer topological invariants such as cohomology rings or homotopy groups; from a computational point of view, Betti numbers are a good trade-off between the ability to capture differences of spaces and tractability. Nevertheless, it might be interesting to employ further topological or geometrical invariants to investigate the expressivity of neural networks in the setting of classification tasks. 

Besides, there is no CPWL function known for which the non-existence of a two hidden layer ReLU neural network computing this function has been proven \citep{hertrich2021towards,haase2023lower}.
Therefore, a further research goal might also be to find topological spaces for which one can show that two hidden layers are not sufficient to have them as a decision region.

It seems straightforward that the construction in Section \ref{lowerbound} can be adapted to neural networks with sigmoidal activation functions in a ``smoothed'' way. Therefore we conjecture that the same lower bound holds for the topological expressivity of neural networks with sigmoidal activation function, which would generalise the lower bound for the zeroth Betti number given in \citet{Bianchini2014OnTC} to all Betti numbers.

Another interesting follow-up research topic would be to investigate the distribution of Betti numbers over the parameter space of the neural network. 
To that end, it is of interest to understand the realization map extensively. Some work in this direction has been done by \cite{grigsby2022functional}.

\acks{The authors would like to thank Christoph Hertrich and Martin Skutella for many valuable discussions and their careful proofreading. Moreover, we thank the various anonymous referees for their remarks. Moritz Grillo and Ekin Ergen are supported by the Deutsche Forschungsgemeinschaft (DFG, German Research Foundation) under Germany's Excellence Strategy --- The Berlin Mathematics Research Center MATH+ (EXC-2046/1, project ID: 390685689).}
\bibliography{bibliography.bib}
\newpage
\appendix\section{Mathematical Background}
\subsection{Polyhedral Geometry}
In this section we recall the definition of a polyhedral complex and introduce some notation related to them. For an introduction to polyhedra, we refer to \citet{Schrijver1986TheoryOL}. Furthermore, we prove two lemmata that we apply in Section~\ref{stabilitysection}.
\label{subsubsection:polyhedra}

    \begin{def1}[Polyhedral complex]
    \label{def:polyhedralComplex}
    A collection of polyhedra $\Po$ is called a \emph{polyhedral complex} if \begin{enumerate}
        \item Every face $F$ of any polyhedra $P \in  \Po$ is also in $\Po$ and
        \item it holds that $P \cap Q \in \Po$ for all $P,Q \in \Po$.
    \end{enumerate}
    There is a poset structure given on $\Po$ by $Q \preceq P :\iff Q$ is a face of $P$ and we call $(\Po, \preceq)$ the \emph{face poset} of the polyhedral complex. Furthermore we denote 
        $\Po_k$ the set of $k$-dimensional polyhedra in $\Po$ and by $\sk_k(\Po)$ the $k$-skeleton of $\Po$. Note that for any polyhedron, the set of all its faces forms a polyhedral complex.
    \end{def1}
     \begin{def1}[Isomorphisms of polyhedral complexes and polytopes]
     Let $\mathcal{P}$ and $\mathcal{Q}$ be polyhedral complexes. A map $\varphi \colon \mathcal{P} \to \mathcal{Q}$ is called an isomorphism if it is an isomorphism of the face posets of $\mathcal{P}$ and $\mathcal{Q}$ and it holds that $\dim(\varphi(P)) = \dim(P)$ for all $P \in \mathcal{P}.$
     
     If $P$ and $Q$ are polytopes we call a map $\varphi \colon P \to Q$ an isomorphism if it is an isomorphism when considering $P$ and $Q$ as polyhedral complexes. 
     \end{def1}
    \begin{def1}
    We call $\varphi \colon \mathcal{P} \to \mathcal{Q}$ an \emph{$\varepsilon$-isomorphism} if 
     it is an isomorphism (of polyhedral complexes) and it holds that $\|\varphi(v)- v\|_2 < \varepsilon$ for all $v \in \ \mathcal{P}_0 $.  
     \end{def1}
    \begin{def1}
    Let $x \mapsto a^Tx +b$ be an affine linear map and $H(a,b) \coloneqq \{x \in \R^d \mid a^Tx + b = 0\}$ the hyperplane given by the kernel. Then we denote the corresponding half-spaces by \[H^1(a,b) \coloneqq \{x \in \R^d \mid a^Tx \geq b\},\] \[H^{-1}(a,b) \coloneqq \{x \in \R^d \mid a^Tx \leq b\}.\] We will also use the notation $H^0(a,b) \coloneqq H(a,b)$. We will simply write $H^s$ for $H^s(a,b)$ whenever $a$ and $b$ are clear from the context.
\end{def1}
    \begin{lemma}
    \label{lemma:PolytopeHyperplane}
        Let $P \subseteq \R^d$ be a polytope, $a\in \R^d$ and  $b \in \R$ such that $P_0 \cap  H(a,b) = \emptyset$. Then for all $\varepsilon > 0$, there is a $\delta > 0$ such that for all $(a',b') \in B^{d+1}_\delta((a,b))$ there are $\varepsilon$-isomorphisms \[\psi^s \colon P \cap H^s(a,b) \to P \cap H^s(a',b')\] for $s \in \{-1,0,1\}$, and it holds that $P_0 \cap  H(a',b') = \emptyset$.
    \end{lemma}
    \begin{proof}
    Let $e \in P_1$ and $\R_e \coloneqq \aff(e)$ be the affine space spanned by $e$. First, assume that \[\R_e \cap H(a,b) \neq \emptyset.\] Since $H(a,b) \cap P_0 = \emptyset$ we know that $\R_e \cap H(a,b) = \{v_{e}^{(a,b)}\}$ with $v_{e}^{(a,b)} \in e^\circ$ or $v_{e}^{(a,b)}\in \R_e \setminus e$, where $e^\circ$ denotes the relative interior of $e$. Let  \[\varepsilon_e \coloneqq \begin{cases}
      \min \{\varepsilon, \frac{1}{2}\inf_{y \in e^\circ} \|y-v_{e}^{(a,b)}\|_\infty\}   &  v_{e}^{(a,b)} \in e^\circ\\
       \min \{\varepsilon, \frac{1}{2}\inf_{y \in \R_e \setminus e} \|y-v_{e}^{(a,b)}\|_{\infty}\}   &  v_{e}^{(a,b)} \in \R_e \setminus e 
    \end{cases}\]

    It is easily verified that the map $(c,d) \mapsto H(c,d) \cap \R_e$ is locally continuous around $(a,b)$ and hence there is a $\delta_e > 0$ such that $\|(a,b) - (a',b') \| < \delta_e$ implies that $\|v_{e}^{(a,b)} -v_{e}^{(a',b')}\|_\infty < \varepsilon_e$ for all $e \in P_1$. On the other hand, if $\R_e \cap H(a,b)  = \emptyset$, then there is a $\delta_e > 0$ such that $e^\circ \cap H(a',b') = \emptyset$. Let $\delta \coloneqq \min_{e \in P_1} \delta_e$.  Note that 
    \[(P \cap H(a',b'))_0 = \{v_{e}^{(a',b')} \mid v_{e}^{(a',b')} \in e^\circ\}\]  and hence $f(v_{e}^{(a,b)}) \coloneqq v_{e}^{(a',b')}$ defines a bijection \[f \colon (P \cap H)_0 \to (P \cap H')_0\] for $(a',b') \in B^{d+1}_\delta((a,b))$. Let $F$ be a face of $P \cap H(a,b)$, then $F = F' \cap H(a,b)$ for some face $F'$ of $P$ and furthermore $F = \conv(\{v_{e}^{(a,b)} \cap e \mid e \preceq F\}$. It now easily follows by induction on the dimension of the face $F$ that $F$ is isomorphic to $\conv(\{v_{e}^{(a',b')} \cap e \mid e \preceq F\}$ and therefore in particular that $P \cap H(a,b)$ is isomorphic to $P \cap H(a',b')$. We can extend $f$ to a bijection $f \colon (P \cap H^s(a,b))_0 \to (P \cap H^s(a',b'))_0$ by $f(v) = v$ for all $v \in P_0 \cap H^s(a,b)$ and by the same arguments we obtain that $P \cap H^s(a,b)$ is isomorphic to $P \cap H^s(a',b')$ for $s\in \{-1,1\}$.
    \end{proof}

    \begin{lemma}
    \label{lemma:PolytopePolytope}
        Let $P\subseteq \R^d$ be a polytope and let $H = \{x \in \R^d \mid a^Tx = b\}$ be a hyperplane. If $P_0 \cap  H = \emptyset$ then for all $\varepsilon > 0$ there is a $\delta > 0$ such that for all polytopes $Q \subseteq \R^d$ with  $\delta$-isomorphisms $\varphi \colon P \to Q$ there are $\varepsilon$-isomorphisms
            \[\gamma^s \colon P \cap H^s \to Q \cap H^s\]
           for $s \in \{-1,0,1\}$ and
        furthermore it holds that $Q_0 \cap  H = \emptyset$.
    \end{lemma}
    \begin{proof}
    Let $e \in P_1$ and let $\partial e = \{u,v\}$. We adopt the notation $e=uv$ 
    and define \[\delta_e \coloneqq 
      \min \{\varepsilon, \frac{1}{2}\inf_{y \in H} \|y-v\|_\infty, \frac{1}{2}\inf_{y \in H} \|y-u\|_\infty\}\] and $\delta \coloneqq \min_{e \in P_1} \delta_e$. Since $P_0 \cap  H = \emptyset$ it holds that $\delta > 0$. Let $\varphi \colon P \to Q$ be a $\delta$-isomorphism. Then it holds $H \cap uv = \{v_e\} \neq \emptyset$ if and only if $H \cap \varphi(u)\varphi(v) = \{v_{\varphi(e)}\} \neq \emptyset$. Note that $(P \cap H)_0 = \{v_e \mid H \cap e \neq \emptyset\}$ and $(Q \cap H)_0 = \{v_{\varphi(e)} \mid H \cap \varphi(e) \neq \emptyset\}$ and hence $f(v_e) \coloneqq v_{\varphi(e)}$ defines a bijection $f \colon (P \cap H)_0 \to (Q \cap H)_0$. The remainder of the proof follows analogously to the proof of Lemma~\ref{lemma:PolytopeHyperplane}.
      
    \end{proof}

\subsection{Topology}
\label{top_background}

In the following, we summarize background knowledge necessary for our purposes that the reader may not have been acquainted with. The content of this subsection can also be found in many algebraic topology textbooks such as~\citet[Chapter~2]{hatcher}.

First, we recall two well-known constructions in topology that yield well-behaved, yet more complex topological spaces.

\begin{def1}
For two topological spaces $X$ and $Y$, the space $X\sqcup Y$ denotes the \emph{disjoint union} of $X$ and $Y$ endowed with the disjoint union topology. Similarly for an arbitrary index set $I$, the set $\bigsqcup_{i\in I}X_i$ denotes the disjoint union of the topological spaces $X_i$ for $i\in I$. If $I$ is a finite set, i.e., $I=\{1,\ldots, q\}$ for a suitable $q\in N$, we also denote this space by $\bigsqcup_{i=1}^qX_i$.
\end{def1}

We also create \emph{product spaces}: For two topological spaces $X$ and $Y$, the product space is the Cartesian product $X\times Y$ endowed with the product topology. Even though it is possible to extend this definition to infinite families of topological spaces as well, this will not be needed for our purposes.

To relate topological spaces $X,Y$ with each other, one often defines \emph{maps}, i.e., continuous functions $f\colon X\to Y$, and investigates properties of such maps. In our work, we use the following special cases of maps:

\begin{def1}
Let $X$, $Y$ be topological spaces.
\begin{enumerate}[label=(\roman*)]
    \item A map $f\colon X\to Y$ is called a \emph{homeomorphism} if it is open and bijective. In this case, we call $X$ and $Y$ \emph{homeomorphic}.
    \item A map $F\colon X\times [0,1]\to Y$ is a \emph{homotopy between functions $f\colon X\to Y$ and $g\colon X\to Y$} if $F(x,0)=f(x)$ and $F(x,1)=g(x)$ for all $x\in X$. In this case, the maps $f$ and $g$ are called \emph{homotopic}, denoted by $f\sim g$.
    \item A map $f\colon X\to Y$ is a \emph{homotopy equivalence} if there exists a map $g\colon Y\to X$ such that $g\circ f$ resp. $f\circ g$ are homotopic to $\mathrm{id}_X$ resp. $\mathrm{id}_Y$. The spaces $X$ and $Y$ are called \emph{homotopy equivalent} in this case.
    \item Let $A\subseteq X$ be a subspace of $X$. A map $r\colon X\to A$ is called a \emph{retraction} if $r(a)=a$ for all $a\in A$. 
    \item A homotopy between a retraction $r\colon X\to A$ (or more precisely, a map $r\colon X \to X$ with $r(X)\subseteq A$ and $r|_A=\mathrm{id}_A$) and the identity map $\mathrm{id}_X$ is called a \emph{deformation retraction}.
\end{enumerate}
\end{def1}

Next, we introduce the notion of homology by giving a sketch of the construction of homology groups.

Let $X$ be a topological space and 
\[\Delta_n=\left\{\sum_{i=0}^n\theta_ix_i\colon x\in \R^{n},\sum_{i=0}^n\theta_i=1, \theta_i\geq 0 \text{ for all $i=0,\ldots, n$}\right\}\]
denote the standard $n$-simplex. Note that the standard $n$-simplex is the convex combination of $n+1$ points $\{p_0, \ldots, p_n\}$. Taking the convex combination of an $n$-subset $\{p_0, \ldots, p_n\}\setminus \{p_i\}$ of these points, one obtains a subspace homeomorphic to the standard $n-1$-simplex, which we call an \emph{$i$-th $n$-face} of the simplex.

The $\Z$-module $C_n$, the group of \emph{$n$-chains}, is defined as the free abelian group generated by continuous maps $\sigma\colon \Delta_n\to X$, called \emph{simplices}. The inclusion $\iota_i\colon \Delta_{n-1}\hookrightarrow\Delta_n$ induces $n-1$-simplices $\sigma_i\coloneqq \sigma\circ\iota_i\colon \Delta_{n-1}\to X$ by the inclusion of the $i$-th $n$-face into the standard simplex. 

The map $\partial_n\colon C_n\to C_{n-1}$, which we call the \emph{boundary map}, is constructed as $\sigma\mapsto \sum_{i=1}^n(-1)^i\sigma_i$ on the generators and by linear extension elsewhere. This yields a \emph{chain complex}

\[\ldots\to C_{n+1}\xrightarrow{\partial_{n+1}} C_n\xrightarrow{\partial_n} C_{n-1}\to\ldots,\]

where we have $\partial_{n-1}\circ\partial_n=0$ for all $n$. Therefore, we can define the \emph{$n$-th singular homology group} as 

\[H_n(X)\coloneqq \altfrac{\ker(\partial_{n})}{\im(\partial_{n+1})}.\]

We list some well-known properties of (singular) homology groups that will be used in our constructions.

\begin{prop}\label{topoez}
Let $n\in \N$.
\begin{enumerate}
    \item (Disjoint union axiom, implication) For any index set $I$ and topological spaces $X_i$ for $i\in I$, it holds that $H_n\left(\bigsqcup_{i\in I}X_i\right)\cong\bigoplus_{i\in I}H_n(X_i)$. 
    \item (Homotopy invariance axiom) Let $X,Y$ be topological spaces and $r \colon X\to Y$ a homotopy equivalence. Then the map $H_n(r)$ is an isomorphism for all $n\in \mathbb{N}$. In particular, it holds that $H_n(Y)\cong H_n(X)$.
    \item (Dimension axiom) $H_n(D^d)=\begin{cases}
    \Z, &n=0\\
    0 &\text{else}
    \end{cases}$
    \item $H_n(S^d)=\begin{cases}
    \Z\oplus \Z &n=d=0\\
    \Z, &n=d\neq 0\text{ or } d\neq n=0\\
    0 &\text{else}
    \end{cases}$

\end{enumerate}
\end{prop}

\begin{obs}\label{annhomo}
Using Proposition~\ref{topoez} and given definitions, one can immediately calculate the homology groups of a $d$-dimensional $k$-annuli:

\[H_n(S^k\times D^{d-k})=H_n(S^k)=\begin{cases}
        \Z\oplus \Z &n=k=0\\
    \Z, &n=k\neq 0\text{ or } k\neq n=0\\
    0 &\text{else}
\end{cases}\]
\end{obs}

To ease our computations for upper bounds, we deviate to another homology theory called \emph{cellular homology} which is defined on a special class of topological spaces called \emph{CW-complexes}. 

\begin{def1}
A Hausdorff space $X$ with a filtration $\emptyset=X_{-1}\subseteq X_0\subseteq \ldots\subseteq \bigcup_{i=1}^dX_d=X$ is a \emph{$d$-dimensional finite CW complex} if the following axioms hold:

\begin{enumerate}[label=(\roman*)]
    \item A subset $A\subseteq X$ is closed in $X$ if and only if $A\cap X_i$ is closed in $X_i$ for all $i\in[d]_0$.
    \item The spaces $X_i$ in the filtration are each called $i$-skeleton. The $i$-skeleton is recursively obtained from $X_{i-1}$ by attaching cells, i.e. we have pushout maps of the form
    \begin{center}
            \begin{tikzcd}
        \bigsqcup_{j\in I_i}S^{i-1} \arrow{r}{\bigsqcup_{j\in I_i}q^j_i} \arrow[hookrightarrow]{d}{} & X_{i-1} \arrow[hookrightarrow]{d}{} \\
 \bigsqcup_{j\in I_i}D^{i}\arrow{r}{\bigsqcup_{j\in I_i}Q^j_i} & X_i
    \end{tikzcd}
    \end{center}
    for finite index sets $I_i$ for $i\in [d]_0$. The maps $q^j_i$ are called \emph{attaching maps} and the maps $Q^j_i$ are called \emph{characteristic maps}. 
\end{enumerate}
For $i\in [d]_0$, the set of path components of $X_i\setminus X_{i-1}$ is called the set of \emph{open $i$-cells}. The set of closures of open $i$-cells are called \emph{closed $i$-cells}. We almost always make use of closed $i$-cells and therefore refer to them simply as \emph{$i$-cells}.
\end{def1}

The non-expert can understand a pushout map as one that simply glues the boundary of an $i$-dimensional cell (that is, a topological disk/polyhedron of dimension $i$) onto the $(i-1)$-skeleton $X_i$. Here, the choice of the attaching maps define the commutative pushout diagram above, while the characteristic maps are those that are ``uniquely'' defined by the attaching maps ``ìn a natural way''. Figure~\ref{fig:CW} illustrates the above definition.

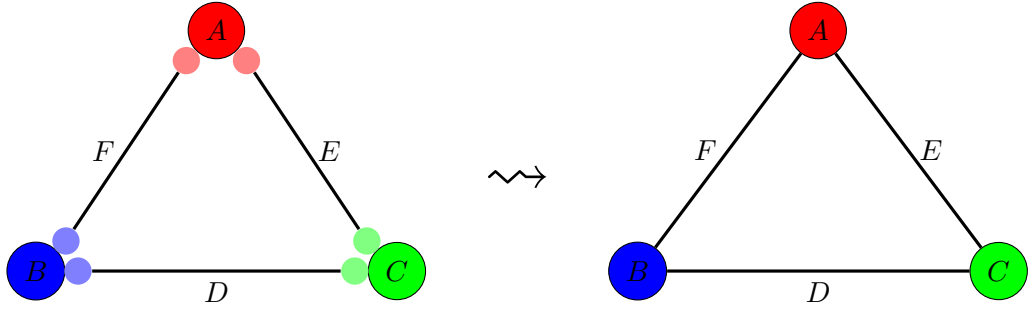
\begin{figure}
    \centering
  \begin{tikzpicture}[scale=0.8]
         \node[shape=circle,draw=black, fill=red] (A) at (0,5) {$A$};
         \node[shape=circle, fill=red!50] (A') at (-0.5,4.5) {};
         \node[shape=circle, fill=red!50] (A'') at (0.5,4.5) {};
         \node[shape=circle,draw=black, fill=blue] (B) at (-3,1) {$B$};
         \node[shape=circle, fill=blue!50] (B') at (-2.5,1.5) {};
         \node[shape=circle, fill=blue!50] (B'') at (-2.3,1) {};
         \node[shape=circle,draw=black, fill=green] (C) at (3,1) {$C$};
         \node[shape=circle, fill=green!50] (C') at (2.5,1.5) {};
         \node[shape=circle, fill=green!50] (C'') at (2.3,1) {};
         \path [very thick] (A') edge node[left] {$F$} (B');
         \path [very thick] (B'') edge node[below] {$D$} (C'');
         \path [very thick] (A'') edge node[right] {$E$} (C');

         \node[] (X) at (5,2.5) {\Huge$\leadsto$};

         \node[shape=circle,draw=black, fill=red] (A) at (10,5) {$A$};
         \node[shape=circle,draw=black, fill=blue] (B) at (7,1) {$B$};
         \node[shape=circle,draw=black, fill=green] (C) at (13,1) {$C$};
         \path [very thick] (A) edge node[left] {$F$} (B);
         \path [very thick] (B) edge node[below] {$D$} (C);
         \path [very thick] (A) edge node[right] {$E$} (C);
    \end{tikzpicture}
    \caption{A CW-complex $X$ of dimension $1$ that is homeomorphic to $S^1$. The darker shaded points constitute the $0$-cells, i.e., we have $X_0=\{A,B,C\}$. The line segments on the left are the $1$-cells. The triangle on the right illustrates $X=X_1$. The lighter shades of colors indicate the attaching maps $q^j_1\colon S^0\to \{A,B,C\}$ for $j\in \{D,E,F\}$.}
    \label{fig:CW}
\end{figure}

A more general definition of CW complexes allow an infinte dimension of cells as well as an (arbitrarily indexed) infinite number of cells in every dimension. For our results, it is sufficient to restrict to our definition, as we will restrict to a CW complex that is a compact topological space without loss of generality:

\begin{lemma}
A CW-complex is finite if and only if it is compact.
\end{lemma}

One can naturally endow finite polyhedral complexes with CW-structures by defining the $i$-cells as the $i$-facets (therefore uniquely defining the filtration), and the attaching maps as those that include each face into the $i$-skeleton of the CW complex for each $i$. This way, the poset structure is compatible with the characteristic maps as well (because faces of polyhedra lie on their topological boundary and the attaching maps $q^j_i$ are injective in this case). One can also observe that any polyhedral complex is homotopy equivalent to a finite CW-complex, see~\ref{lemma:polyhedraltoCW}.

Our motivation to endow polyhedral complexes with CW-structures is to use \emph{cellular homology}, which, given the natural CW-structure of a polyhedral complex, allows to compute homology groups conveniently. Among other advantages, we will rely on cellular homology for an induction on the number of polyhedra. 

\begin{def1}
Let $X$ be a finite CW-complex. The \emph{cellular chain complex} $(C_{i})_{i\in \mathbb{N}}$ of $X$ is given by free abelian groups $C_i\cong \mathbb{Z}^{|I_i|}$ that are generated by the $i$-cells of $X$. The boundary maps, which are given by a construction using attachment maps in the general case, can be greatly simplified for our purposes: The boundary map $\partial_i\colon C_i\to C_{i-1}$ is defined by the incidence matrix $\Delta\in\Z^{|I_{i-1}|\times|I_i|}$ between $i$ and $(i-1)$-cells, that is, it is given by entries \[\delta_{jk}=\begin{cases}
    1 & \text{the $(i-1)$-cell $j$ lies in the boundary of the $i$-cell $k$}\\
    0 & \text{else}.
\end{cases}\]
The \emph{$i$-th cellular homology} $H^{cell}_i(X)$ of $X$ is defined by the homology of the cellular chain complex $(C_{i})_{i\in \mathbb{N}}$, that is, we have
\[H^{cell}_i(X)=\altfrac{\ker(\partial_i)}{\im(\partial_{i-1})}.\]
\end{def1}

It is well-known that on CW-complexes, cellular homology groups coincide with singular homology groups. Therefore, we may make use of cellular homology groups in order to compute Betti numbers.

\section{Lower Bound}
{In this section we provide formal proofs for the statements made in sections \ref{lowerbound} and additional lemmata that are used in these proofs. For the sake of completeness, we also recall the statements we prove.}

\subsection{Proof of Lemma \ref{lemma:babycubes}}
\label{foldingnetwork}
\begin{def1}
We define the one hidden layer ReLU neural network $h^{(1,m,d)}$ in the following way:
The neurons $\{v_{i,j}\}_{i=0,\ldots,m-1,j=1,\ldots,d}$ in the hidden layer are given by: 	\begin{itemize}
		\item $v_{0,j}(x) = \max \{0,m x_j\}, j=1,\ldots,d$
		\item $v_{i,j}(x) = \max \{0,2m (x_j-i/m)\}, j=1,\ldots,d \ i=1,\ldots,m-1$
	\end{itemize}
	and the output neurons by: 
	 $h^{(1,m,d)}_{j}(x) = \sum_{i=0}^{m-1}(-1)^i \cdot v_{i,j}(x)$. 
\end{def1}

\begin{lemma}
\label{lemma:1hiddenlayer_folding}
Let $d,m \in \N$ with $m>1$. Then  \begin{enumerate}
	\item $h^{(1,m,d)}(W^{(1,m,d)}_{(i_1,\ldots,i_{d})})= [0,1]^{d}$
    \item $\pi_j \circ h^{(1,m,d)} _{|W^{(1,m,d)}_{(i_1,\ldots,i_{d})}}(x_1,\ldots,x_{d})= \left\{
	\begin{array}{ll}
		m\cdot x_j -(i_j-1) & i_j \text{ odd} \\
		-m\cdot x_j + i_j  & \, i_j \text{ even} \\
	\end{array} \right.$

\end{enumerate}
for all $(i_1,\ldots,i_{d}) \in [m]^{d}$.
\end{lemma}
\begin{proof}Throughout this proof we denote $W^{(1,m,d)}_{(i_1,\ldots,i_{d})}$ by $W_{(i_1,\ldots,i_{d})}$ and $h^{(1,m,d)}$ by $h$.
	 We prove that $h$ satisfies the second property. The first property then follows immediately from the second since \[\pi_j \circ h_{|W_{(i_1,\ldots,i_{d})}}(W_{(i_1,\ldots,i_{d})})= \left[m\cdot \frac{(i_j-1)}{m}-(i_j-1), 	m\cdot \frac{i_j}{m}-(i_j-1)\right] = [0,1]\] if $i_j$ is odd and \[\pi_j \circ h_{|W_{(i_1,\ldots,i_{d})}}(W_{(i_1,\ldots,i_{d})})= \left[m \cdot \frac{(i_j-1)}{m}+ i_j, m \cdot \frac{i_j}{m}+ i_j\right] = [0,1]\] if $i_j$ is even. \newline
	 Let $j \in \{1,\ldots,d\}$ and $x \in W_{(i_1,\ldots,i_{d})}$, so in particular $x_j \in \left[\frac{(i_j-1)}{m}, \frac{i_j}{m}\right]$. Since $i \geq i_j $ implies \mbox{$2m (x_j-i/m) \leq 0$}, it follows that $v_{i,j}(x) = 0$ for all $i \geq i_j$. Similarly, $i < i_j$ implies \mbox{$2m (x_j-i/m) \geq 0$}, and therefore it follows that  $v_{i,j}(x) = 2m (x_j-i/m)$ for all $i <i_j$. Hence
	 \[h_{j}(x) = \sum_{i=0}^{i_j-1}(-1)^i \cdot v_{i,j}(x) = mx_j + \sum_{i=1}^{i_j-1}(-1)^i \cdot 2m (x_j-i/m).\]If $i_j$ is even, then
	 	 \begin{align*}h_j(x)&= mx_j + \sum_{i=1}^{i_j/2 -1} 2m (x_j-2i/m) -\sum_{i=1}^{i_j/2}  2m (x_j-(2i-1)/m)\\&= mx_j - 2(i_j/2 -1)- 2m (x_j-(2i_j/2-1)/m) \\&= mx_j - i_j+2 - 2mx_j + 2i_j -2 \\&= -mx_j + i_j 
	\end{align*}
If $i_j$ is odd, then
\begin{align*} h_{j}(x) &= mx_j + \sum_{i=1}^{(i_j-1)/2} 2m (x_j-2i/m) -\sum_{i=1}^{(i_j-1)/2}  2m (x_j-(2i-1)/m)\\&= mx_j - 2(i_j-1/2) \\&= mx_j -(i_j-1).
\end{align*}
\end{proof}
\babycubes*
\begin{proof}
     
We apply induction over $L$. The base case has already been covered by Lemma~\ref{lemma:1hiddenlayer_folding}. Assume that there exists a NN $h^{(L-1,(m_1,\ldots,m_{L-1}),d)}$ that satisfies the desired properties and define \mbox{$h^{(L,\mathbf{m},d)} = h^{(1,m_L,d)} \circ h^{(L-1,(m_1,\ldots,m_{L-1}),d)}$}. Let $j \in \{1,\ldots,d\}$ and $x \in W^{(L,\mathbf{m},d)}_{(i_1,\ldots,i_{d})}$. Define $i_j^{(1)}\coloneqq\left\lfloor \frac{m_L\cdot(i_{j}-1)}{M} \right\rfloor +1$. It holds that $\left[\frac{(i_j-1)}{M}, \frac{i_j}{M}\right] \subset \left[\frac{(i_j^{(1)}-1)}{m_L}, \frac{i^{(1)}_j}{m_L}\right] $. 
	\begin{casesp} 
	\item  $i_j^{(1)}$ odd. Then by Lemma~\ref{lemma:1hiddenlayer_folding}:
\begin{align*}
	h^{(1,m_L,d)}_j\left(\tfrac{(i_j-1)}{M}\right)&= m_L \cdot \frac{(i_j-1)}{M} - (i_j^{(1)} -1)  = \frac{m_L \cdot (i_j-1)}{M} - \left\lfloor \frac{m_L \cdot(i_{j}-1)}{M} \right\rfloor \\&= \frac{(i_j-1)\bmod (M/m_L)}{(M/m_L)}    
\end{align*}
 and 
\begin{align*}
	h^{(1,m_L,d)}_j\left(\tfrac{i_j}{M}\right) &= m_L \cdot \frac{i_j}{M} - (i_j^{(1)} -1)  = \frac{m_L \cdot (i_j)}{M} - \left\lfloor \frac{m_L \cdot (i_{j}-1)}{M} \right\rfloor \\&= \frac{(i_j-1) \bmod (M/m_L)+1}{(M/m_L)}.    
\end{align*}

 Define $i_j^{(L-1)}\coloneqq ((i_j-1) \bmod (M/m_L)) +1$. Then it holds that \[h^{(1,m_L,d)}_j(x) \in \left[\frac{m_L\cdot(i_j^{(L-1)}-1)}{M},\frac{m_L\cdot i_j^{(L-1)}}{M}\right].\]
 Moreover, $i_j^{(L-1)}= ((i_j-1) \bmod (M/m_L)) +1$ is odd if and only if $i_j$ is odd because $\frac{M}{m_L}$ is an even number.
\begin{casesp}
\item $i_j^{(L-1)}$ (and therefore $i_j$) is odd. Then follows with the induction hypothesis:
	\begin{align*}
	    h_j^{(L,\mathbf{m},d)}(x) &=  h_j^{(L-1,(m_1,\ldots,m_{L-1}),d)}(h_j^{(1,m_L,d)}(x)) \\&=\frac{M}{m_L}\cdot ( h_j^{(1,m_L,d)}(x)) - (i_j^{(L-1)} -1) \\&=  \frac{M}{m_L} \cdot ( m_Lx - (i_j^{(1)} -1)) - (i_j^{(L-1)}-1)   \\&=   Mx - \frac{M}{m_L}\cdot\left\lfloor \frac{m_L \cdot (i_{j}-1)}{M} \right\rfloor -  ((i_j-1) \bmod (M/m_L))  \\&= Mx - (i_j-1).
	\end{align*}
	
	\item $i_j^{(L-1)}$ (and therefore $i_j$) is even. Then follows with the induction hypothesis:
	\begin{align*}
	    h_j^{(L,\mathbf{m},d)}(x) &=  h_j^{(L-1,(m_1,\ldots,m_{L-1},d)}(h_j^{(1,m_L,d)}(x)) \\&=-\frac{M}{m_L}\cdot ( h_j^{(1,m_L,d)}(x)) + i_j^{(L-1)} \\&=  -\frac{M}{m_L} \cdot ( m_Lx - (i_j^{(1)} -1)) + i_j^{(L-1)}   \\&=   -Mx + \frac{M}{m_L}\cdot\left\lfloor \frac{m_L \cdot (i_{j}-1)}{M} \right\rfloor +  ((i_j-1) \bmod (M/m_L) +1)  
	    \\&=   -Mx + i_j-1 +1  \\&= -Mx + i_j.
	\end{align*}

\end{casesp}
	\item $i_j^{(1)}$ even. Then by Lemma \ref{lemma:1hiddenlayer_folding}:
\begin{align*}
	h^{(1,m_L,d)}_j(\tfrac{(i_j-1)}{M}) &= -m_L \cdot \frac{(i_j-1)}{M} + i_j^{(1)}  = {-}\frac{m_L\cdot(i_j-1)}{M} + \left\lfloor \frac{m_L\cdot(i_{j}-1)}{M} \right\rfloor +1 \\&= 1{-}\frac{(i_j-1) \bmod (M/m_L)}{M/m_L}     
\end{align*}
 and 
\begin{align*}
	h^{(1,m_L,d)}_j(\tfrac{i_j}{M}) &= -m_L \cdot \frac{i_j}{M} + i_j^{(1)}  = -m_L\cdot\frac{i_j}{M} + \left\lfloor m_L\cdot\frac{(i_{j}-1)}{M} \right\rfloor +1 \\&= 1-\frac{(i_j-1) \bmod (M/m_L))-1}{M/m_L}    
\end{align*}
 
 Define $i_j^{(L-1)}\coloneqq \frac{M}{m_L}-((i_j-1) \bmod (M/m_L))$. Then it holds that \[h^{(1,m_L,d)}_j(x) \in \left[\frac{m_L\cdot(i_j^{(L-1)}-1)}{M},\frac{m_L\cdot(i_j^{(L-1)})}{M}\right].\]
 Moreover, $i_j^{(L-1)}$ is even if and only if $i_j$ is odd, once more because $\frac{M}{m_L}$ is an even number.
\begin{casesp}
\item $i_j^{(L-1)}$ is odd (i.e., $i_j$ even). Then follows with the induction hypothesis:

	\begin{align*}
    h_j^{(L,\mathbf{m},d)}(x) 
    &=  h_j^{(L-1,(m_1,\ldots, m_{L-1}),d)}(h_j^{(1,m_L,d)}(x)) \\
    &=\frac{M}{m_L}\cdot ( h_j^{(1,m_L,d)}(x)) - (i_j^{(L-1)}-1)  \\
    &= \frac{M}{m_L} \cdot ( -m_Lx + i_j^{(1)}) - (i_j^{(L-1)}-1)   \\
    &=  -Mx + \frac{M}{m_L}\cdot\left\lfloor \frac{m_L\cdot(i_{j}-1)}{M} \right\rfloor +\frac{M}{m_L} -  \left( \frac{M}{m_L}-((i_j-1) \bmod M/m_L)-1\right)\\
    &= -Mx + i_j
	\end{align*}

	\item $i_j^{(L-1)}$ is even (i.e., $i_j$ odd). Then follows with the induction hypothesis:
	\begin{align*}
	    h_j^{(L,\mathbf{m},d)}(x) 
        &=  h_j^{(L-1,(m_1,\ldots, m_{L-1}),d)}(h_j^{(1,m_L,d)}(x)) \\
        &=-\frac{M}{m_L}\cdot ( h_j^{(1,m_L,d)}(x)) + i_j^{(L-1)} \\
        &=  -\frac{M}{m_L} \cdot ( -m_Lx + i_j^{(1)}) + i_j^{(L-1)}   \\&=   Mx - \frac{M}{m_L}\cdot\left(\left\lfloor \frac{m^L(i_{j}-1)}{M} \right\rfloor +1\right) + \frac{M}{m_L}-((i_j-1) \bmod (M/m_L)+1) +1\\  
        &=Mx - (i_j-1).
	\end{align*}
\end{casesp}
\end{casesp}
This concludes the proof for all cases.

\end{proof}
\subsection{Proof of Lemma \ref{lemma:mirrored_cuts}}
\begin{lemma}
	\label{lemma:finalLayer}
Let $d,w \in \N$ and  \[R_q =\{x \in \R^{d}:x_1,\ldots,x_{d} >0, \frac{q}{4w}<\|x\|_1<\frac{q+1}{4w}\}.\] Then $\sgn(\hat{g}(R_q)) = (-1)^q$ for all $ q=0,\ldots,w-1$ and $\hat{g}(x)=0$ for all $x \in [0,1]^{d}$ with $\|x\|_1\geq\frac{1}{2}.$ 
\end{lemma} 
\begin{proof}
	Let $q \in \{0,\ldots,w-1\}$ and $x \in R_q$. Note that $\hat{g}_0(x)=\1x$ for all $q \in \{0,\ldots,w-1\}.$ 

    \begin{casesp}

    \item  
     $\1x < (2q+1)/4w.$ This implies $\hat{g}_i(x) = 0 \ \forall q >i$ and \mbox{$g_i(x) = 2(\1x-((2i-1)/4w))$} for all $ 1 < i \leq q$ and therefore \[\hat{g}(x) = \sum_{i=0}^{q} (-1)^{i}\hat{g}_i(x) = x_1 + \sum_{i=1}^{q} (-1)^{i}2(\1x-((2i-1)/4w)).\] 
      \begin{casesp}
      \item 
   If $q$ is even, then it holds: \begin{align*}
		\hat{g}(x) &=  \1x + \sum_{i=1}^{q/2} 2(\1x-((2(2i)-1)/4w))  - \sum_{i=1}^{q/2}  2(\1x-((2(2i-1)-1)/4w))
		\\&=  \1x + \sum_{i=1}^{q/2} 2(\1x-((4i-1)/4w))  - \sum_{i=1}^{q/2}  2(\1x-((4i-3)/4w)) \\&= \1x -q/2w > 0
	\end{align*}

\item
If $q$ is odd, then it holds: \begin{align*}
	\hat{g}(x) &=  \1x + \sum_{i=1}^{(q-1)/2} 2(\1x-((4i-1)/4w))  - \sum_{i=1}^{(q+1)/2}  2(\1x-((4i-3)/4w)) \\&= \1x -2(q-1)/4w - 2(\1x-((4(q+1)/2)-3)/4w)) \\&=-(\1x) -2(q-1)/4w + (4(q+1) -6)/4w\\&=-(\1x) +q/2w < 0
\end{align*}

	 \end{casesp}
    
    \item 
	 $\1x \geq (2q+1)/4w$. This implies $\hat{g}_i(x) = 0 \ \forall i >q+1$ and \[g_i(x) = 2(\1x-((2i-1)/4w)) \]  for all $1 < i \leq q+1$ and therefore \[\hat{g}(x) = \sum_{i=0}^{q+1} (-1)^{q}\hat{g}_i(x) = x_1 + \sum_{i=1}^{q+1} (-1)^{i}2(\1x-((2i-1)/4w)).\] 
	\begin{casesp}
	
	\item 
	 
If $q$ is even, then it holds: \begin{align*}
	\hat{g}(x) &=  \1x + \sum_{i=1}^{q/2} 2(\1x-((4i-1)/4w))  - \sum_{i=1}^{q/2+1}  2(\1x-((4i-3)/4w)) \\&= \1x -2q/4w - 2(\1x-((4(q/2+1)-3)/4w)) \\&=
	-(\1x) -q/w +2(2q+1)/4w \\&= -(\1x) +(q+1)/2w >0
\end{align*}

\item
If $q$ is odd, then it holds: \begin{align*}
	\hat{g}(x) &=  \1x + \sum_{i=1}^{(q+1)/2} 2(\1x-((4i-1)/4w))  - \sum_{i=1}^{(q+1)/2}  2(\1x-((4i-3)/4w)) \\&= \1x -(q+1)/2w  < 0
\end{align*}

and hence $\sgn(\hat{g}(x)) = (-1)^q \  \forall x \in R_q \ \forall
q=1,\ldots,w-1$.
\end{casesp}
\end{casesp}

Let $x \in [0,1]^{d}$ with $\1x\geq\frac{1}{2}$. 	
\begin{casesp}
\item 
 $w$ even. Then \begin{align*}
    \hat{g}(x) &= \sum_{q=0}^{w+1} (-1)^q \cdot \hat{g}_q(x) \\&=  \1x -(\1x -\frac{1}{2}) + \sum_{q=1}^{w/2}  \hat{g}_{2_q}(x)-\hat{g}_{2_{q-1}}(x) \\&=\frac{1}{2} + \sum_{q=1}^{w/2}  2(\1x-(2\cdot 2q-1)/4w) - 2(\1x-(2\cdot (2q-1))-1)/4w)) \\&= \frac{1}{2} + \sum_{q=1}^{w/2} 2(-(4q-1)/4w + (4q-3)/4w)  \\&= \frac{1}{2} + \sum_{q=1}^{w/1} - 1/2w \\&= 0
\end{align*}

\item  $w$ odd. Then
\begin{align*}
    \hat{g}(x) &= \sum_{q=0}^{w+1} (-1)^q \cdot \hat{g}_q(x) \\&= \hat{g}_0(x)  -\hat{g}_w(x) +  \hat{g}_{w+1}(x) + \sum_{q=1}^{w-1/2}  \hat{g}_{2_q}(x)-\hat{g}_{2_{q-1}}(x) \\&= \1x - 2(\1x-(2w-1)/4w)+ \left(\1x-\frac{1}{2}\right) + \sum_{q=1}^{w-1/2} - 1/w \\&= (2w-1)/2w -\frac{1}{2}  + \sum_{q=1}^{w-1} - 1/2w \\&= 1 - 1/2w - \frac{1}{2}- \left(\frac{1}{2} -1/2w\right) \\&=0
\end{align*}
\end{casesp}
\end{proof}
\mirroredcuts*
\begin{proof}
    Let the affine map $t \colon [0,1]^d \to [0,1]^d$ be given by $x \mapsto (1-x_1,\ldots,1-x_{d-1},x_d)$ and let $\hat{g}$ be the 1-hidden layer neural network from Lemma~\ref{lemma:finalLayer}. We prove that the neural network $g \coloneqq \hat{g}\circ t$ satisfies the assumptions. Let $q \in \{0,\ldots,w-1\}$ and $x \in R_q$. Then \[\|(1,1,\ldots,1,0)-t(x)\|_1 =\|(1,1,\ldots,1,0)-(1-x_1,\ldots,1-x_{d-1},x_d)\|_1 = \|x\|_1.\] Since $g(x) = g \circ t(x) = \hat{g}(t(x))$, Lemma \ref{lemma:finalLayer} implies that $\sgn(g(R_q)) = (-1)^q$. Analogously follows that $g(x)=0$ for all $x \in [0,1]^{d}$ with $\|(1,1,\ldots,1,0)-x\|_1\geq\frac{1}{4}.$ 
\end{proof}

\subsection{Proof of Proposition \ref{prop:homeo1}}
\begin{restatable}{lemma}{onionrings}\label{lemma:onionrings}
Let  $g^{(w,d)}$ be the NN from Lemma \ref{lemma:mirrored_cuts} and $C$ the set of cutting points. Define \[R_{q,c} \coloneqq B_{(q+1)/(4w \cdot M)}^d(c)\setminus \overline{B_{q/(4w \cdot M)}^d(c)}\] for a cutting point $c \in C$ and $q\in \{1,\ldots, w\}$. Then 
\begin{enumerate}
    \item $g^{(w,d)} \circ h^{(L,\mathbf{m},d)}(R_{q,c}) \subseteq (-\infty,0]$ for $q$ odd, 
    \item $g^{(w,d)} \circ h^{(L,\mathbf{m},d)}(R_{q,c}) \subseteq [0,\infty)$ for $q$ even
    \item $x \notin \bigcup\limits_{q \in [w],c\in C}R_{q,c}$ implies $g^{(w,d)} \circ h^{(L,\mathbf{m},d)}(x)=0$.
\end{enumerate} 

In particular, $g^{(w,d)} \circ h^{(L,\mathbf{m},d)}(x)=0$ for all $ x\in \partial R_{q,c}$.
\end{restatable}
\begin{proof}
	By definition of $c$ being a cutting point, there exist odd numbers $i_1,\ldots,i_{d-1} \in [M]$ and an even number $i_d \in [M]$ such that $c =(\frac{i_1-1}{M},\ldots,\frac{i_d-1}{M})$. Let $x \in [0,1]^d$ with  $\|x-c\|_\infty \leq \frac{1}{M}$, then either $x_j \in \left[\frac{i_j-2}{M},\frac{i_j-1}{M}\right]$ or $x_j \in \left[\frac{i_j-1}{M},\frac{i_j}{M}\right]$. Let $J^+ :=\{j \in [d]:x_j-c_j \leq 0\}$ be the set of indices $j$ such that $x_j \in \left[\frac{i_j-1}{M},\frac{i_j}{M}\right]$ and $J^- := [d] \setminus J^+$. Let $y = h^{(L,\mathbf{m},d)}(x) \in [0,1]^d$. Then it follows with Lemma \ref{lemma:babycubes} that  \begin{align*}
		\1y &= \sum_{j\in J^+} M \cdot x_j -(i_j-1) + \sum_{j\in J^-} -M \cdot x_j +(i_j-1) \\&= \sum_{j\in J^+} M \cdot (x_j-c_j +c_j) -(i_j-1) + \sum_{j\in J^-} -M \cdot (x_j-c_j +c_j) +(i_j-1) \\&= \sum_{j\in J^+} M \cdot (x_j-c_j) +  \sum_{j\in J^+} M \cdot c_j  -(i_j-1) \\&+ \sum_{j\in J^-} -M \cdot (x_j-c_j ) + \sum_{j\in J^-} -M \cdot c_j +(i_j-1) \\&= \sum_{j\in J^+} M \cdot (x_j-c_j) + \sum_{j\in J^-} -M \cdot (x_j-c_j) \\&= M \cdot  \sum_{j=1}^{d}|x_j-c_j| \\&= M \cdot \|x-c\|_1
	\end{align*}

If $x \in R_{q,c}$, then in particular  $\|x-c\|_\infty \leq \frac{1}{M}$ and thus: \[\1y = M \cdot \|x-c\|_1 < M \cdot \frac{q+1}{M \cdot 4w} = \frac{q+1}{4w}\] and \[\1y = M \cdot \|x-c\|_1 > M \cdot \frac{q}{M \cdot 4w} = \frac{q}{4w}.\] If $q$ is even, with Lemma \ref{lemma:mirrored_cuts} it follows that $g(y) \in (0,\infty)$ and therefore $g \circ h^{(L,\mathbf{m},d)}(x) = g(y)$. The case where $q$ is odd follows analogously and this concludes the first two cases.\newline
If $x$ is not in any $R_{q,c}$, If x is not in any $R_{q,c}$, then either $x \in \partial R_{q,c}$ for some cutting point $c$ and some $q$ or it holds that $\|x - c\|_1 \geq
\frac{1}{2M}$ for every cutting point $c$. In the first case it follows directly from the above shown
that $g \circ  h^{(L,m,d)}
(x)) = 0$, since that $g \circ  h^{(L,m,d)}$
is continuous. In the second case there exists a cutting point $c$ such that $\|x-c\|_\infty \leq \frac{1}{M}$, since for every $x_j$ either $\lfloor M \cdot x_j \rfloor$ or $\lceil M \cdot x_j \rceil$ is even. Thus $\1 h^{(L,\mathbf{m},d)}(x) = M \cdot \|x-c\|_1 \geq M \cdot  \frac{1}{4 \cdot M} = \frac{1}{4}$ and therefore it follows with Lemma \ref{lemma:mirrored_cuts} that $g \circ h^{(L,\mathbf{m},d)}(x)=0$, which concludes the proof. 

\end{proof}

\homeoone*
\begin{proof}
    We observe that the sets $Y_{d,w}\cap \overline{B^d}_{1/{4M}}(c)$ are disjoint for cutting points $c$ because we have $||c-c'||_1\geq \frac{2}{M}$ for any two distinct  cutting points $c,c'$. 
    Therefore, the sets $Y_{d,w}\cap \overline{B}^d_{1/{4M}}(c)$ are pairwise disjoint for $c\in C$.
    Since \[\bigsqcup\limits_{c\in C}Y_{d,w}\cap \overline{B^d_{1/{4M}}(c)}=Y_{d,w},\] the number of cutting points of the interior is $\frac{M^{\cdot (d-1)}}{2^{d-1}}\cdot\left(\frac{M}{2}-1\right)$ and the number of the cutting points on the boundary is $\frac{M^{\cdot (d-1)}}{2^{d-2}}$ by Observation~\ref{obs:evenpts}, it suffices to show that \[Y_{d,w} \cap \overline{B^d_{1/{4M}}(c)}\cong\bigsqcup\limits_{i=1}^{\left\lceil\frac{nw}{2}\right\rceil}S^{d-1}\times D^1\] for every  $c\in C\cap \Int([0,1]^d)$ and $Y_{d,w} \cap \overline{B^d_{1/{4M}}}(c)\cong\bigsqcup\limits_{i=1}^{\left\lceil\frac{w}{2}\right\rceil} D^d$ for every $c\in C\cap \partial[0,1]^d$. 

    By Lemma~\ref{lemma:onionrings}, we can see that for every $c\in C\cap \Int([0,1]^d)$, we have
    \begin{align*}
Y_{d,w} \cap \overline{B^d_{1/{4M}}(c)}&=\bigsqcup\limits_{1\leq q\leq w \text{ odd} } B_{q/(w \cdot 4M)}^d(c)\setminus \overline{B_{(q-1) /(w \cdot 4M)}^d(c)}\\&\cong \bigsqcup\limits_{1\leq q\leq w \text{ odd} }S^{d-1}\times [0,1]\\&=\bigsqcup\limits_{q=1}^{\left\lceil\frac{w}{2}\right\rceil}S^{d-1}\times [0,1],        
    \end{align*}
    as well as for every $c\in C\cap \partial([0,1]^d)$, we have
    \[Y_{d,w} \cap \overline{B^d_{1/{4M}}(c)}\cong \bigsqcup\limits_{1\leq q\leq w \text{ odd} } \left(B_{q/(w \cdot 4M)}^d(c)\setminus \overline{B_{(q-1) /(w \cdot 4M)}^d(c)}\right)\cap [0,1]^d\cong \bigsqcup\limits_{q=1}^{\left\lceil\frac{w}{2}\right\rceil} D^d,\]
    proving the claim. 
\end{proof}

\subsection{Proof of Lemma \ref{lemma:enoughspace}}
\enoughspace*
\begin{proof}
We adopt the notation $c^{(k)}\coloneqq(1,1,\ldots,1,0) \in \R^k$ throughout.

We first show that for all $x \in [0,1]^k$,
\begin{equation}\label{idek}
\| x-c^{(k)}\|_1 \leq \frac{1}{4} \Rightarrow g^{(w_{k-2},k-1)} \circ p_{k-1}(x) = 0.
\end{equation}

Let $x \in [0,1]^k$ with $g^{(w_{k-2},k-1)} \circ p_{k-1}(x) \neq 0$. Lemma \ref{lemma:mirrored_cuts} implies that \mbox{$\|p_{k-1}(x) -c^{(k-1)} \|_1 < \frac{1}{4}$.} Therefore we have $\frac{1}{4} > |\pi_{k-1}\circ p_{k-1}(x)-0| = |\pi_{k-1}(x)-0|=x_{k-1}$ which also means $|x_{k-1} -1| > \frac{1}{4}$ and thus $\|x-c^{(k)}\|_1 > \frac{1}{4}$. \newline
Note that by Lemma~\ref{lemma:mirrored_cuts} it suffices to show that $f^{(w_1,...,w_{k-1})} \circ p_{k-1}(x) = 0$ for all $x$ with {$\|x-c^{(k)}\|_1 \leq \frac{1}{4}$.} We prove this by induction over $k$. The base case has already been covered since $g^{(w_1,2)} = f^{w_1}$. Furthermore \begin{align*}
    f^{(w_1,...,w_{k-2})} \circ p_{k-1} &= (f^{(w_1,...,w_{k-3})} \circ p_{k-2} + g^{(w_{k-2},k-1)}) \circ p_{k-1}
    \\& = f^{(w_1,...,w_{k-3})} \circ p_{k-2} \circ p_{k-1} + g^{(w_{k-2},k-1)} \circ p_{k-1}
    \\&= f^{(w_1,...,w_{k-3})} \circ p_{k-2}  + g^{(w_{k-2},k-1)} \circ p_{k-1}
\end{align*} 
and thus the induction hypothesis and (\ref{idek}) imply that $f^{(\mathbf{w})} \circ p_{k-1}(x) = 0$ for $x$ with \mbox{$\|x-c^{(d)}\|_1 \leq \frac{1}{4}$.}
\end{proof}

\subsection{Proof of Lemma~\ref{lemma:homeo}}
\begin{lemma}\label{lemma:commute}
    The following diagram commutes: 
    \begin{center}
    \begin{tikzcd}
    {[0,1]^k} \arrow{rr}{h^{(L,\mathbf{m},k)}} \arrow[swap]{d}{p_{k-1}} & & {[0,1]^k} \arrow{d}{p_{k-1}} \arrow{drr}{{f^{\mathbf{w}}}\circ p_{k-1}}\\
     {[0,1]^{k-1}} \arrow[swap]{rr}{h^{(L,\mathbf{m},k-1)}} & &{[0,1]^{k-1}} \arrow[swap]{rr}{f^{\mathbf{w}}} && \R
  \end{tikzcd}
  \end{center}
\end{lemma}
\begin{proof}
In order to show that the left half of the diagram commutes, we prove that
\[(\pi_j\circ h^{(L,\mathbf{m},k-1)}\circ p_{k-1})(x)=(\pi_j\circ p_{k-1}\circ h^{(L,\mathbf{m},k)})(x)\]
for every $j\in \{1,\ldots, k-1\}$ and $x\in [0,1]^k$. For any $x\in [0,1]^k$, there exist indices $i_1,\ldots,i_k$ such that $x=(x_1,\ldots, x_k)\in W^{(L,\mathbf{m},k)}_{(i_1,\ldots,i_{k})}$. Moreover, if $x\in W^{(L,\mathbf{m},k)}_{(i_1,\ldots,i_{k})}$, we have $p_{k-1}(x)\in W^{(L,\mathbf{m},k-1)}_{(i_1,\ldots,i_{k-1})}$ because \[p_{k-1}\left(W^{(L,\mathbf{m},k)}_{(i_1,\ldots,i_{k})}\right)=p_{k-1}\left(\prod_{j=1}^{k} \left[\frac{(i_j-1)}{M}, \frac{i_j}{M}\right]\right)=\prod_{j=1}^{k-1} \left[\frac{(i_j-1)}{M}, \frac{i_j}{M}\right].\]

We use this observation combined with Lemma~\ref{lemma:babycubes}, assuming that $i_j$ is odd:
\begin{align*}
(\pi_j\circ h^{(L,\mathbf{m},k-1)}\circ p_{k-1})(x)&=M\cdot (p_{k-1}(x))_j-(i_j-1)\\
                                        &=M\cdot x_j-(i_j-1)\\
                                        &=(\pi_j \circ h^{(L,\mathbf{m},k)})(x)\\
                                        &=(\pi_j\circ p_{k-1}\circ h^{(L,\mathbf{m},k)})(x),
\end{align*}
as claimed. The case where $i_j$ is even follows analogously.
\end{proof}

\homeo*
\begin{proof}
 For $k=2$ it holds that $f^{w_1}=g^{(w_1,2)}$ and therefore the claim holds trivially. Now let $k \geq 3$. Since $f^{(w_1,...,w_{k-1})} = f^{(w_1,\ldots,w_{k-2})} \circ p_{k-1} + g^{(w_{k-1},k)}$ and the spaces \[(f^{(w_1,\ldots,w_{k-2})} \circ p_{k-1}\circ h^{(L,\mathbf{m},k)})^{-1}((-\infty,0])\] and  $(g^{(w_{k-1},k)}\circ h^{(L,\mathbf{m},k)})^{-1}((-\infty,0])$ are disjoint by Lemma~\ref{lemma:enoughspace}, it follows that 

    \begin{align*}
        &(f^{(w_1,...,w_{k-1})}\circ h^{(L,\mathbf{m},k)})^{-1}((-\infty,0]) \\&= (( f^{(w_1,\ldots,w_{k-2})} \circ p_{k-1} + g^{(w_{k-1},k)})\circ h^{(L,\mathbf{m},k)})^{-1}((-\infty,0])
        \\&=  (f^{(w_1,\ldots,w_{k-2})} \circ p_{k-1} \circ h^{(L,\mathbf{m},k)} + g^{(w_{k-1},k)}\circ h^{(L,\mathbf{m},k)})^{-1}((-\infty,0])
        \\&=  (f^{(w_1,\ldots,w_{k-2})} \circ p_{k-1} \circ h^{(L,\mathbf{m},k)})^{-1}((-\infty,0]) \sqcup (g^{(w_{k-1},k)}\circ h^{(L,\mathbf{m},k)})^{-1}((-\infty,0])
        \\&=  (f^{(w_1,\ldots,w_{k-2})} \circ h^{(L,\mathbf{m},k-1)} \circ p_{k-1})^{-1}((-\infty,0]) \sqcup (g^{(w_{k-1},k)}\circ  h^{(L,\mathbf{m},k)})^{-1}((-\infty,0])&
         \\&= X_{k-1} \times [0,1]  \sqcup Y_{k,w},  
    \end{align*} where the second last equality is due to Lemma \ref{lemma:commute}.
\end{proof}
\subsection{Proof of Theorem~\ref{theorem:main}}
\begin{lemma}\label{lemma:inductionbase}
    The space $Y_{d,w}\coloneqq (g^{(w,d)}\circ h^{(L,\mathbf{m},d)})^{-1}((-\infty,0])$ satisfies 
    
    \begin{enumerate}[label=(\roman*)]
        \item $H_0(Y_{d,w})\cong\Z^{p+p'}$,
        \item $H_{d-1}(Y_{d,w})\cong\Z^{p}$,
        \item $H_k(Y_{d,w})=0$ for $k\geq d$ 
    \end{enumerate} with $p=\frac{M^{(d-1)}}{2^{d-1}}\cdot\left(\frac{M}{2}-1\right) \cdot\left\lceil\frac{w}{2}\right\rceil$ and $p'=\frac{M^{(d-1)}}{2^{d-2}}\cdot\left\lceil \frac{w}{2}\right\rceil$
\end{lemma}
\begin{proof}
    Follows directly from Observation~\ref{annhomo} and Proposition~\ref{prop:homeo1} and the disjoint union axiom (Proposition~\ref{topoez}).
\end{proof}

\main*
\begin{proof}
    We consider the map $F\coloneqq f^{(\mathbf{w})}\circ h^{(L,\mathbf{m},d)}$ that was previously constructed (Lemma \ref{lemma:homeo}) and let $X_d = F^{-1}((-\infty,0))$. For $d=2$, the statement is identical to Lemma~\ref{lemma:inductionbase}. Indeed, we have
    \begin{align*}
       2\cdot\frac{\left(\frac{M}{2}+1\right)^{3}-1}{M}-\frac{M}{2}-2&=\left(\frac{M}{2}+1\right)^2+\frac{M}{2}+1-\frac{M}{2}-2\\&=\frac{M}{2}\left(\frac{M}{2}+1\right).     
    \end{align*}

    Let $d\geq 3$. Using Proposition~\ref{prop:homeo1}, we see that

\begin{align}
        H_k(X_d)\cong H_k(X_{d-1} \sqcup Y_{d,w_{d-1}})\cong H_k(X_{d-1})\oplus \prod_{i=1}^{p_d} H_k(S^{d-1})\oplus \prod_{i=1}^{p'_d}H_k(D^d)
    \end{align}

    and therefore
    \begin{align}\label{eqn:homology}
        \beta_k(X_d)=\beta_k(X_{d-1})+ \sum_{i=1}^{p_d}\left(\beta_k(S^{d-1})\right)+ \sum_{i=1}^{p'_d}\beta_k(D^d)
    \end{align}
    where $p_d=\frac{M^{ d-1}}{2^{d-1}}\cdot\left(\frac{M}{2}-1\right) \cdot\left\lceil\frac{w_{d-1}}{2}\right\rceil$ and $p'_d=\frac{M^{ d-1}}{2^{d-2}}\cdot\left\lceil \frac{w_{d-1}}{2}\right\rceil$.
Fix some $k\in \mathbb{N}$. For different values of $k$, we obtain the claims:
\begin{enumerate}[label=(\roman*)]
    \item For $k=0$, equation~(\ref{eqn:homology}) implies 
    \[  \beta_0(X_d)=\beta_0(X_{d-1})+p_d+p'_d=\sum_{i=2}^d\frac{M^{(i-1)}}{2^{i-1}}\cdot\left(\frac{M}{2}+1\right)\cdot\left\lceil \frac{w_{i-1}}{2}\right\rceil   
    \]
    
    \item For $k \leq d-1$, we have $\beta_{d-1}(X_{d})=0$ and therefore \[\beta_{d-1}(X_d)=p_d=\left(\frac{M}{2}-1\right) \cdot\frac{M^{d-1}}{2^{d-1}}\cdot\left\lceil \frac{w_{d-1}}{2}\right\rceil. \]
    For $0<k<d-1$, we have $\beta_k(X_d)=\beta_k(X_{d-1})$, i.e., the claim is satisfied by induction.
    \item Finally for $k\geq d$, we observe that all summands of~(\ref{eqn:homology}) vanish.

    Lastly, the norm of the weights of the map $h^{(L,\mathbf{m},d)}$ are bounded from above by  $\max_{\ell = 1, \ldots L}2\frac{n_\ell}{d}$ and the norm of the weights in the output layer $f^{(\mathbf{w})}$ are bounded from above by $1 \leq \max_{\ell = 1, \ldots L}2\frac{n_\ell}{d}$. 

\end{enumerate}
\end{proof}
\subsection{Closure}
In this section we modify the construction slightly to obtain a lower bound for $\beta_k(F^{-1}((-\infty,0])$ additionally to the lower bound for $\beta_k(F^{-1}((-\infty,0))$
\label{closure}
\begin{figure}
   \centering
\begin{tikzpicture}[scale = 0.55]

    \fill[lightgray] (8-1.95,8,8) -- (8,8-1.95,8) -- (8,8,8-1.95) -- cycle;
    \draw[] (8-1.95,8,8) -- (8,8-1.95,8) -- (8,8,8-1.95) -- cycle;
    \fill[darkgray] (8-1.55,8,8) -- (8,8-1.55,8) -- (8,8,8-1.55) -- cycle;
    \draw[] (8-1.55,8,8) -- (8,8-1.55,8) -- (8,8,8-1.55) -- cycle;
    \fill[lightgray] (8-0.95,8,8) -- (8,8-0.95,8) -- (8,8,8-0.95) -- cycle;
    \draw[] (8-0.95,8,8) -- (8,8-0.95,8) -- (8,8,8-0.95) -- cycle;
    \fill[darkgray] (8-0.55,8,8) -- (8,8-0.55,8) -- (8,8,8-0.55) -- cycle;
    \draw[] (8-0.55,8,8) -- (8,8-0.55,8) -- (8,8,8-0.55) -- cycle;

    \fill[lightgray] (8,1.95,8) -- (8-1.95,0,8) -- (8,0,8) -- (8,0,0) -- (8,1.95,0) -- cycle;
   \draw (8,1.95,8) -- (8-1.95,0,8)-- (8,1.95,0) -- cycle;
    \fill[darkgray] (8,1.55,8) -- (8-1.55,0,8) -- (8,0,8) -- (8,0,0) -- (8,1.55,0) -- cycle;
    \draw(8-1.55,0,8) -- (8,1.55,8) -- (8,1.55,0);
    \fill[lightgray] (8,0.95,8) -- (8-0.95,0,8) -- (8,0,8) -- (8,0,0) -- (8,0.95,0) -- cycle;
     \draw(8-0.95,0,8) -- (8,0.95,8) -- (8,0.95,0);
    \fill[darkgray] (8,0.55,8) -- (8-0.55,0,8) -- (8,0,8) -- (8,0,0) -- (8,0.55,0) -- cycle;
    \draw (8-0.55,0,8) -- (8,0.55,8) -- (8,0.55,0);
    
    \fill[darkgray] (0,0,8) -- (8-1.95,0,8) -- (8,1.95,8) --(8,8-1.95,8) --(8-1.95,8,8) -- (0,8,8) -- cycle;
    \fill[darkgray] (8-1.95,8,8) -- (0,8,8) -- (0,8,0) -- (8,8,0) -- (8,8,8-1.95) -- cycle;
    \fill[darkgray] (8,8,0) -- (8,8,8-1.95)--(8,8-1.95,8)  -- (8,1.95,8)-- (8,1.95,0) -- cycle;
    
    \foreach \x in{2,4,6}
{   \draw[gray,very thin] (0,\x ,8) -- (8,\x ,8);
    \draw[gray,very thin] (\x ,0,8) -- (\x ,8,8);
    \draw[gray,very thin] (8,\x ,8) -- (8,\x ,0);
    \draw[gray,very thin] (\x ,8,8) -- (\x ,8,0);
    \draw[gray,very thin] (8,0,\x ) -- (8,8,\x );
    \draw[gray,very thin] (0,8,\x ) -- (8,8,\x );
}
    \foreach \x in{0,8}
{   \draw (0,\x ,8) -- (8,\x ,8);
    \draw (\x ,0,8) -- (\x ,8,8);
    \draw (8,\x ,8) -- (8,\x ,0);
    \draw (\x ,8,8) -- (\x ,8,0);
    \draw (8,0,\x ) -- (8,8,\x );
    \draw (0,8,\x ) -- (8,8,\x );
}

\draw[name path=y-axis,->] (0,0,8) -- (0,9,8) node[above]{$x_2$};
\draw[->, name path=x-axis] (0,0,8) -- (9,0,8) node[right]{$x_1$};
\draw[->, name path=z-axis] (8,0,8) -- (8,0,-2) node[right]{$x_3$};

\draw[dashed] (0,0,8) -- (0,0,0);
\draw[dashed] (0,0,0) -- (0,8,0);
\draw[dashed] (0,0,0) -- (8,0,0);
\filldraw [black] (0,8,8) circle (2pt)node[anchor=east]{$1$};
\filldraw [black] (8,0,8) circle (2pt)node[anchor=north]{$1$};
\filldraw [black] (0,0,0) circle (2pt)node[anchor=north]{$1$};
\end{tikzpicture} 
\caption{The last hidden layer of the modified $F$ in the proof of Theorem~\ref{cor:exactformulaclosure}; Adding a small constant $b$ in the output layer in order to not have full dimensional ``$0$-regions".}
   \end{figure}
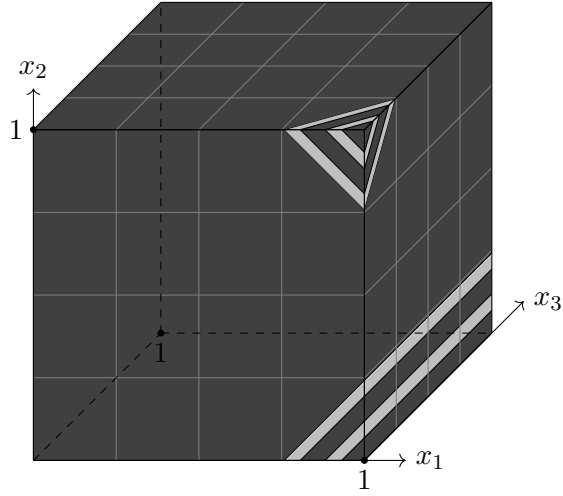 

\closure*

\begin{proof}
Let $F'$ be the neural network constructed in Theorem~\ref{theorem:main}. We want to modify this neural network to obtain a neural network $F$ such that it holds that 
\begin{itemize}
    \item $F^{-1}((-\infty,0))$ and $F'^{-1}((-\infty,0))$ are homeomorphic,
    \item $F^{-1}((-\infty,0))$ is homotopy equivalent to its closure $\overline{F^{-1}((-\infty,0))}$ and
    \item $\overline{F^{-1}((-\infty,0))}= F^{-1}((-\infty,0])$.
\end{itemize}
    In order to guarantee the third property, we need to ensure that $F^{-1}(\{0\}) = \partial F^{-1}((-\infty,0))$. Therefore, we transform the full-dimensional $0-$regions into slightly positive, but still constant regions by adding a small constant $b$, i.e., by simply setting $F(x) = F'(x) + b$, where {$b= \min_{k=2,\ldots,d}\{\frac{1}{8w_k\cdot M}\}$}. 
    
    Next, we argue that $F^{-1}((-\infty,0))$ and $F'^{-1}((-\infty,0))$ are homeomorphic, since adding $b$ also just makes the annuli in $F'^{-1}((-\infty,0))$ thinner. Let $k \in [d]$ and $A$ be an $k-$annuli in $F'^{-1}((-\infty,0))$ and $A_k = p_k(A)$ be its projection onto the first $k$ coordinates. It follows that there is a cutting point $c\in \R^k$ such that $A_k =  B_{q/(4w_k \cdot M)}^k(c)\setminus \overline{B_{(q-1)/(4w_k \cdot M)}^k(c)}$ for a suitable $q=1,...,w_k$. Since $b \leq \{\frac{1}{8w_k\cdot M}\}$ it follows that \[A'_{k,c,q} \coloneqq B_{(q-2b)/(4w_k \cdot M)}^k(c)\setminus \overline{B_{(q-1+2b)/(4w_k \cdot M)}^k(c)}\] is also a $k$-annuli and revisiting the proof of Lemma \ref{lemma:finalLayer} reveals that $F(A'_{k,c,q} \times \R^{d-k}) = (-\frac{1}{4w_k},-b)$ and hence $F(A'_{k,c,q} \times \R^{d-k}) =(-\frac{1}{4w_k} + b, 0)$ since $\frac{1}{8w_k} \geq b$. We conclude that for every $k$-annuli in $F'^{-1}((-\infty,0))$ there is a ``shrunk" $k$-annuli in $F^{-1}((-\infty,0))$ and the ``shrinking" for every annuli is an homeomorphism. The case for the disks at the boundary of the unit cube follows analogously.  Since $F^{-1}((-\infty,0))$ is the disjoint union of these disks and annuli and it clearly holds that $F(x)>0$ for all $x \in F'^{-1}((0,\infty))$, it follows that $F^{-1}((-\infty,0))$ and $F'^{-1}((-\infty,0))$ are homeomorphic. Since disks and annuli are homotopy equivalent to their closures, it follows that $F^{-1}((-\infty,0))$ is homotopy equivalent to $\overline{F^{-1}((-\infty,0))}$, proving the claim.

\end{proof} 
\section{Stability}
In this section we aim to prove Proposition~\ref{prop:robust_to_pertubation}.
Before we prove the stability of our construction, we prove stability for a wider range of neural networks. Throughout this section we will use definitions and statements from Section~\ref{subsubsection:polyhedra}.
First, we define the realization map, that maps, for a given architecture, a vector of weights to its corresponding neural network.
\label{stabilitysection}
\begin{def1}[The realization map]
\label{def:realization map}
Let $(n_0, \ldots n_{L+1})$ be an architecture. Its \emph{ corresponding parameter space} is given by \mbox{$\R^D \cong \bigoplus^{L+1}_{\ell=1} \R^{(n_{\ell-1} +1) \times n_{\ell}}$},  where the vector space isomorphism is given by $p \mapsto (A^{(\ell)}(p),b^{(\ell)}(p))_{\ell=1,...,L+1}$ for $A^{(\ell)}(p) \in \R^{n_{\ell-1} \times n_\ell}, b^{(\ell)}(p) \in \R^{n_\ell}$ and $ \ell =1,\ldots,L+1$. 
For a polyhedron  $K\subseteq \R^d$, we define $\Phi \colon \R^D \to C(K)$ to be the \emph{realization map} that assigns to a vector of weights the function the corresponding neural network computes, i.e., \[\Phi(p) \coloneqq T_{L+1}(p) \circ \sigma_{n_L} \circ T_{L}(p) \circ \cdots \circ \sigma_{n_1} \circ T_1(p)\]
where $T_\ell(p) \colon \R^{n_{\ell-1}} \to \R^{n_\ell}$, $x \mapsto A^{(\ell)}(p) x +b^{(\ell)}(p)$.

Furthermore let \[\Phi^{(\ell)}(p) \coloneqq T_\ell(p) \circ \sigma_{n_\ell} \circ \cdots \circ \sigma_{n_1} \circ T_0(p)\] and  \[\Phi^{(i,\ell)}(p) \coloneqq \pi_i \circ T_{\ell}(p) \circ \cdots \circ \sigma_{n_1} \circ T_0(p),\] which we call the map computed at neuron $(i,\ell)$.   We denote the points of non-linearity introduced by the $i$-th neuron in the $\ell$-th layer by \[\Tilde{H}_{i,\ell}(p) \coloneqq H\left(A_i^{(\ell)}(p),b_i^{(\ell)}(p)\right)\]
\end{def1}
Now we define a sequence of polyhedral complexes associated to a neural network. Every polyhedral complex $\Po^{(i,\ell)}$ in this sequence corresponds to a refinement of the input space $K$ such that the CPWL map computed at a neuron $(i,\ell)$ in the neural network as well as the CPWL map computed at any neuron $(j,k)$ with $(j,k)$ lexicographically smaller than $(i,\ell)$ is affine linear on all polyhedra in $\Po^{(i,\ell)}$. Moreover, it is a refinement of the previous polyhedral complexes $\Po^{(i-1,\ell)}$ in this sequence that results by intersecting $\Po^{(i-1,\ell)}$ with the pullbacks of $\Tilde{H}^s_{i,\ell}(p), s \in \{-1,0,1\}$ by $\Phi^{(\ell)}(p)$.
\begin{def1}[Canonical polyhedral complex (\cite{grigsby2022local})]
\label{def:canonicalPC}
Let $p\in \R^D$ be a vector of weights. Recall that $\Phi(p)$ is the corresponding neural network.

    We recursively define polyhedral complexes $\mathcal{P}^{(i,\ell)}(p, K)$ by \[\mathcal{P}^{(0,0)}(p, K) \coloneqq \{F  \mid F \text{ is a face of } K\}\] and 
     \[\mathcal{P}^{(i,\ell)}(p, K) \coloneqq \{R \cap (\Phi^{(\ell-1)}(p))^{-1}(\Tilde{H}^s_{i,\ell}(p))\mid R \in \mathcal{P}^{(i-1,\ell)}(p, K), s \in \{-1,0,1\} \}\] 
     for $i =2,\ldots n_\ell, \ell=1,\ldots L$ and
    \[\mathcal{P}^{(1,\ell)}(p, K) \coloneqq \{R \cap (\Phi^{(\ell-1)}(p))^{-1}(\Tilde{H}^s_{i,\ell}(p, K))\mid R \in \mathcal{P}^{(n_{\ell-1},\ell-1)}(p, K), s \in \{-1,0,1\} \}\] 
    for $\ell=1,\ldots L$.
    
\end{def1}

Note that for all $j \leq i$, it holds that $\Phi^{(j,\ell)}(p)$  is affine linear on $R$ for each $R \in \mathcal{P}^{(i,\ell)}(p)$ and we denote this affine linear map by $\Phi_{\mid R}^{(j,\ell)}(p)$. For $\ell \in [L], i \in [n_{\ell}]$ and $R \in \mathcal{P}_d^{(i,\ell)}(p, K)$ we denote the points of non-linearity in the region $R$  with respect to the first $\ell-1$ layer map introduced by the $i$-th neuron in the $\ell$-th layer by \[H_{i,\ell,R}(p) \coloneqq (\Phi_{\mid R}^{(\ell-1)}(p))^{-1}(\Tilde{H}_{i,\ell}(p)) = H\left(A_i^{(\ell)}(p)\left(\Phi_{\mid R}^{(\ell-1)}(p)(x)\right),b_i^{(\ell)}(p) \right).\]
For the sake of simplification we set $\Po^{(0,\ell)}(p,K) \coloneqq \Po^{(n_{\ell -1},\ell-1)}(p,K).$ Furthermore, since for $R \in \mathcal{P}_d^{(i-1,\ell)}(p,K), F \in \mathcal{P}^{(i-1,\ell)}(p,K), F \preceq R$ it holds that \[F \cap H^s_{i,\ell,R}(p) = F \cap (\Phi^{(\ell-1)}(p))^{-1}(\Tilde{H}^s_{i,\ell}(p))\] due to the continuity of the function $\Phi(u)$, we have that  \[\mathcal{P}^{(i,\ell)}(p) = \{F \cap H^s_{i,\ell,R}(p) \mid R \in \mathcal{P}_d^{(i-1,\ell)}(p), F \in \mathcal{P}^{(i-1,\ell)}(p), F \preceq R, s \in \{-1,0,1\} \}\] 
We call $\Po(p,K) \coloneqq \Po^{(n_L,L)}(p,K)$ the \emph{canonical polyhedral complex} of $\Phi(u)$ (with respect to $K$). We omit $p$ respectively $K$ whenever $p$ respectively $K$ is clear from the context.

In the following, we find a sufficient condition for a neural network, which is, that the points of non-linearity introduced at any neuron $(i,\ell)$ do not intersect the vertices of the polyhedral complex $\mathcal{P}^{(i,\ell)}(p, K)$, to compute a ``similar looking map on a polytope $K$'' $\Phi(u)$ for ``close enough'' parameters $u$. Note that the boundedness of the polyhedron is required because any perturbation of two parallel unbounded $(d-1)-$faces results in new intersection patterns in the corresponding polyhedral complex.

     \begin{def1}
     \label{def:combstable}
     Let $K\subseteq \R^d$ be a polytope and $\Phi(p) \colon K \to \R$ be a ReLU neural network of architecture $(n_0, \ldots n_{L+1})$. Then we call $\Phi(p)$ \emph{combinatorially stable} (with respect to $K$) if for every $ \ell \in [L+1], i \in [n_\ell]$ and all $R \in \Po_{d}^{(i-1,\ell)}(p,K)$ it holds 
     \begin{enumerate}[label=(\roman*)]
         \item $\dim(H_{i,\ell,R}(p)) = d-1$ and 
         \item $H_{i,\ell,R}(p) \cap R_0 = \emptyset$. 
     \end{enumerate}
     \end{def1}

     We will now prove that this condition is indeed sufficient.
    \begin{prop}
    \label{prop:combstable}
    Let $K$ be a polytope and $\Phi(p) \colon K \to \R$ be a stable ReLU neural network of architecture $(n_0, \ldots n_{L+1})$. Then for every $\varepsilon > 0$, there is a an open set $U \subseteq \R^D$ such that for every $u \in U$ there is an $\varepsilon$-isomorphism $\varphi_u \colon \mathcal{P}(p,K) \to \mathcal{P}(u,K)$.
    \end{prop}

    \begin{proof}
    We will prove the following stronger statement by induction on the indexing of the neurons.
    \begin{claim}
    For every $ \ell \in [L+1]$, $i \in [n_\ell]$ and every $\varepsilon > 0$, there is a $\delta>0$ such that for all $u \in B_\delta^{\|\cdot\|_\infty}(p)$ there is an $\varepsilon$-isomorphism $\varphi_u^{(i,\ell)} \colon \Po^{(i,\ell)}(p) \to \Po^{(i,\ell)}(u)$.
    \end{claim}
    The induction base is trivially satisfied. \newline
    So we assume that the statement holds for $(i-1,\ell)$. For simpler notation we denote $\varphi_u^{(i-1,\ell)}$ by $\varphi_u$ and $H_{i,\ell,R}(p)$ by $H_{R}(p)$.
    Let $\varepsilon > 0$ and $F \in \Po^{(i-1,\ell)}(p)$. There is an $R \in \Po_d^{(i-1,\ell)}(p)$ such that $F \preceq R$.  In the following we wish to find a $\delta_F >0$ such that there are $\varepsilon$-isomorphisms \[\varphi^{(i,\ell)}_{(u,R,s)} \colon F \cap H^s_{R}(p) \to \varphi_u(F) \cap H^s_{\varphi_u(R)}(u)\] for $s\in \{-1,0,1\}$ and all $u \in B_{\delta_F}^{\|\cdot\|_\infty}(p)$.

    Since $\Phi(p)$ is stable, we obtain by Lemma~\ref{lemma:PolytopePolytope} a $\delta_2 > 0$ such that for all $\delta_2$-isomorphisms $\varphi \colon F \to Q$ there are $\frac{\varepsilon}{3}$-isomorphisms $\gamma^s \colon F \cap H^s_{R}(p) \to \varphi(F) \cap H^s_{R}(p)$. By the induction hypothesis we obtain $\delta_1>0$ such that for all $u \in B_{\delta_1}^{\|\cdot\|_\infty}(p)$ there is an $\delta_2$-isomorphism 
    \[\varphi_u \colon \Po(p)^{(i-1,\ell)} \to \Po(u)^{(i-1,\ell)}\] and hence we obtain $\frac{\varepsilon}{3}$-isomormorphisms \[\gamma^{(s,F)} \colon F \cap  H^s_{R}(p) \to \varphi_u(F) \cap H^s_{R}(p).\] 
    
    Let $H_{\varphi_u(R)}(u,p) \coloneqq H_{R}(u_{1,1},\ldots u_{i-1,\ell},p_{i,\ell},\ldots p_{n_{L+1},L+1})$ with $u_{j,k},p_{j,k} \in \R^{n_k}$ being the parameters associated to the $j$-th neuron in the $k$-th layer. Again, for simpler notation, let the affine maps $\Phi_{\mid R}^{(\ell-1)}(p)$  be given by $x \mapsto Mx + c$ and $\Phi_{\mid \varphi_u(R)}^{(\ell-1)}(u)$ by $x \mapsto Nx + d$ and the non-linearity points introduced by the $i$-th neuron in the $\ell$-th layer by $\Tilde{H}_{i,\ell}(p) = H(a,b)$. Then we have that 
    \[H_{R}(p) = H(a^TM,a^Tc + b)\] and 
    \[H_{\varphi_u(R)}(u,p) = H(a^TN,a^Td + b).\] 
    
    By Lemma~\ref{lemma:PolytopePolytope} we know that $(\varphi_u(F))_0 \cap H_{R}(p) = \emptyset$ and hence by Lemma~\ref{lemma:PolytopeHyperplane} there is a $\delta_3>0$ such that there are $\frac{\varepsilon}{3}$-isomorphisms $\psi^s \colon \varphi_u(F) \cap H^s_{R}(p) \to \varphi_u(F) \cap H^s(y,z)$ for all \mbox{$(y,z) \in B^{d+1}_{\delta_3}((a^TM,a^Tc+b))$}. Let $C \coloneqq n_{\ell-1}\cdot\max\limits_{i=1,...,n_{\ell-1}}\{a_i\}$ and $u \in \R^D$ with $\|u-p\|_\infty < \frac{\delta_3}{C}.$ Then we have that
    \begin{align*}
        \| (a^TM, a^Tc + b) - (a^TN, a^Td + b) \|_\infty &= \max_{i=1,\ldots,d}\bigg\{\sum_{j=1}^{n_{\ell-1}} a_j (m_{ij}-n_{ij}),\sum_{j=1}^{n_{\ell-1}}a_j(c_j-d_j)\bigg\} \\&<\max_{i=1,\ldots,d}\bigg\{\sum_{j=1}^{n_{\ell-1}} a_j \frac{\delta_3}{C},\sum_{j=1}^{n_{\ell-1}}a_j\frac{\delta_3}{C}\bigg\} \\& < \delta_3
    \end{align*}
    and hence there are $\frac{\varepsilon}{3}$-isomorphisms \[\psi^{(s,F)} \colon \varphi_u(F) \cap H^s_{R}(p) \to \varphi_u(F) \cap H^s_{\varphi_u(R)}(u,p)\]

    By Lemma~\ref{lemma:PolytopeHyperplane} we know that $(\varphi_u(F))_0 \cap H_{\varphi(R)}(u,p) = \emptyset$ and hence by the same lemma there is a $\delta_4>0$ such that there are $\frac{\varepsilon}{3}$-isomorphisms $\alpha^s \colon \varphi_u(F) \cap H^s_{\varphi(R)}(u,p) \to \varphi_u(F) \cap H^s(y,z)$ for all $(y,z) \in B^{d+1}_{\delta_4}((a^TN,a^Td+b))$. Let $a' \in \R^{n_{\ell-1}},b' \in \R$ such that $\Tilde{H}_{i,\ell}(u) = H(a'^T,b')$. Then we have that \[H_{\varphi_u(R)}(u) = H( a'^TN,a'^Td + b').\]
    Let $E \coloneqq n_{\ell-1}\cdot\max\limits_{i,j=1,...,n_{\ell-1}}\{n_{ij},d_j\}$ and $u \in \R^D$ with $\|u-p\|_\infty < \frac{\delta_5}{E}.$ Then we have that

     \begin{align*}
        \| (a'^TN,a'^Td + b') - (a^TN, a^Td&+ b) \|_\infty\\
        &= \max_{i=1,\ldots,d}\bigg\{\sum_{j=1}^{n_{\ell-1}} n_{ij} (a_j'-a_j),\big(\sum_{j=1}^{n_{\ell-1}}d_{j}(a_j'-a_j)\big) + (b_j'-b_j) \bigg\} \\&<\max_{i=1,\ldots,d}\bigg\{\sum_{j=1}^{n_{\ell-1}} n_{ij} \frac{\delta_5}{E},\sum_{j=1}^{n_{\ell-1}}n_{ij}\frac{\delta_5}{E}\bigg\} \\& < \delta_4
    \end{align*}
    and hence there are  $\frac{\varepsilon}{3}$-isomorphisms \[\alpha^{(s,F)} \colon \varphi_u(F) \cap H^s_{\varphi(R)}(u,p) \to \varphi_u(F) \cap H^s_{\varphi_u(R)}(u).\]

    Let $\delta_{F} \coloneqq \min\{\delta_2,\frac{\delta_4}{C},\frac{\delta_5}{E}\}$, then for all $u \in B_{\delta_{F}}^{D,\|\cdot\|_\infty}(p)$ there is an $\varepsilon$-isomorphism \[\varphi^{(i,\ell)}_{(u,F,s)} \colon F \cap H^s_{R}(p) \to \varphi_u(F) \cap H^s_{\varphi_u(R)}(u)\] given by \[\varphi^{(i,\ell)}_{(u,F,s)} = \alpha^{(s,F)} \circ \psi^{(s,F)} \circ \gamma^{(s,F)}.\] 
    Lastly, let $\delta = \min\{\delta_F \mid F \in \Po^{(i-1,\ell)}(p)\}$. Since every element of $\Po^{(i,\ell)}(p)$ is of the form $F \cap H^s_{R}(p)$, it now remains to show that the map $\varphi_u^{(i,\ell)} \colon \Po^{(i,\ell)}(p) \to \Po^{(i,\ell)}(u)$ defined by \[\varphi_u^{(i,\ell)}(F \cap H^s_{R}(p)) \coloneqq \varphi_u(F) \cap H^s_{\varphi_u(R)}(u)\]
    is an $\varepsilon$-isomorphism for all $u \in B_\delta^{\|\cdot\|_\infty}(p).$ Since $\varphi_u$ and $\varphi^{(i,\ell)}_{(u,F,s)}$ are bijections, the same holds for $\varphi_u^{(i,\ell)}$. Furthermore let $G \preceq F \cap H^s_{R}(p)$, then there is a $G' \preceq F$ and a $s' \in \{0,s\}$ such that $G = G' \cap H^{s'}_{R}(p)$. Since $\varphi_u$ is an isomorphism by the induction hypothesis, it follows that 

    \[\varphi_u^{(i,\ell)}(G' \cap H^{s'}_{R}(p)) = \varphi_u(G') \cap H^{s'}_{\varphi_u(R)}(u) \preceq \varphi_u(F) \cap H^{s}_{\varphi_u(R)}(u)\]

    and hence $\varphi_u^{(i,\ell)}$ is an $\varepsilon$-isomorphism as claimed. 
    \end{proof}
    Taking the hyperplanes where the output layer equals zero into account, we define a topological stable parameter. 
    \begin{def1}
    \label{def:topstable}
    Let $K$ be a polytope and $\Phi(p) \colon K \to \R$ be a ReLU neural network of architecture $(n_0, \ldots,n_L,1)$. Then we call $\Phi(p)$ \emph{topologically  stable} if it is combinatorially stable (with respect to $K$), and for all $R \in \Po_{d}^{(n_L,L)}(p,K)$ it holds that 
     \begin{enumerate}[label=(\roman*)]
         \item $\dim(H_{1,L+1,R}(p)) = d-1$ and
         \item $H_{1,L+1,R}(p) \cap R_0 = \emptyset$. 
     \end{enumerate}
    \end{def1}
    We now prove that topologically stable is the right definition for our purposes, that is, finding an open set $U \subseteq \R^d$ with $p\in U$ such that the sublevel set of $\Phi(u)$ is homeomorphic to the sublevel set $\Phi(p)$ for all $u \in U$.

    \topstable*
    \begin{proof}
    Let 
    \begin{align*}
        \Po^-(p) &\coloneqq \{F \cap H^s_{1,L+1,R}(p) \mid R \in \mathcal{P}_d(p), F \in \mathcal{P}(p), F \preceq R, s \in \{-1,0\} \}\\ &= \{P \cap F^{-1}((-\infty,0]) \mid P \in \Po\}
    \end{align*}
    be the polyhedral complex consisting of all maximal subpolyhedron of $\Po(p)$ where $\Phi(p)$ takes on non-negative values. Analogously to the proof of Proposition~\ref{prop:combstable} we obtain a $\delta>0$ such that $\Po^-(p,K)$ and $\Po^-(u,K)$ are isomorphic as polyhedral complexes and hence in particular there is a homeomorphism $\varphi \colon |\Po^-(p,K)| \to |\Po^-(u,K)|$ for all $u \in B_\delta(p)$, where $|\Po^-(p,K)|$ denotes the support of $\Po^-(p,K)$. This concludes the proof since $|\Po^-(p,K)| = K \cap \Phi(p)^{-1}((-\infty,0])$ and  $|\Po^-(u,K)| = K \cap \Phi(u)^{-1}((-\infty,0])$.
\end{proof}
Having this at hand, we can finally show the stability of the constructed neural network in Theorem~\ref{cor:exactformulaclosure} for the lower bound of the topological expressive power.
\robust*
\begin{proof}
 Let $p \in \R^D$ such that $\Phi(p)=F$ from Theorem~\ref{cor:exactformulaclosure}. Then, since $\Phi(p)$ is topologically stable with respect to any cube it follows by Proposition~\ref{prop:topstable} that there is an open set in $\R^D$ containing $u$ such that $\Phi(u)^{-1}((-\infty,0)) \cap K$ is homeomorphic to $\Phi(p)^{-1}((-\infty,0)) \cap K$ for all $u \in U$, where $K$ is the unit cube.
\end{proof}

Using the results from above, we can even show that if $p$ is topologically stable, then also $\Phi(p)^{-1}((-\infty,0))$ is homeomorphic to $\Phi(u)^{-1}((-\infty,0))$ for all $u$ in an open set $U\subseteq \R^D$.
\begin{prop}
     For every topologically stable network $\Phi(p)$, there is a $\delta >0$ such that for all $u \in B_\delta(p)$, it holds that $K \cap \Phi(p)^{-1}((-\infty,0))$ is homeomorphic to $K \cap \Phi(u)^{-1}((-\infty,0))$.
\end{prop}
\begin{proof}
    We adapt the notation from the proof of Proposition~\ref{prop:topstable} and we know that $|\Po^-(p,K)|$ and $|\Po^-(u,K)|$ are homeomorphic.    

    We wish to show that $|\Po^-(p,K)|^\circ = K^\circ \cap \Phi(p)^{-1}((-\infty,0])$. Due to the continuity of $\Phi(u)$ it holds that $K^\circ \cap \Phi(p)^{-1}((-\infty,0)) \subseteq |\Po^-(p)|^\circ$. Let $x$ be chosen such that $\Phi(p)(x) = 0$, i.e., $x \in |\Po^-(p)| \setminus \left(K^\circ \cap \Phi(p)^{-1}((-\infty,0))\right)$.  Since $\Po(u)$ is a pure polyhedral complex, there is a full-dimensional polyhedron $R \in \Po_d(u)$ such that $x \in R$. It follows that $x \in H_{1,L+1,R}(p) \cap R$ with $\dim(H_{1,L+1,R}(p)) = d-1$. 
    
    Assume for sake of contradiction that $\dim(H^1_{1,L+1,R}(p) \cap R) < d$, then there would be a face $F \preceq R$ such that $F \subseteq H_{1,L+1,R}(p)$, which is a contradiction to $H_{1,L+1,R}(p) \cap R_0 = \emptyset$. Therefore, the space $H^1_{1,L+1,R}(p) \cap R$ is full-dimensional. As a result, $\Phi(u)$ takes on exclusively positive values on $(H^1_{1,L+1,R}(p) \cap R)^\circ \neq \emptyset$ and hence for every open subset $U \subseteq \R^d$ with $x \in U$, it holds that $U \cap \Phi(p)^{-1}((0,\infty)) \neq \emptyset$. Thus, $x \notin |\Po^-(p)|^\circ$ and hence \[|\Po^-(p)|^\circ = K^\circ \cap \Phi(p)^{-1}((-\infty,0)).\] 
    
    Since $\Po^-(p)$ and $\Po^-(u)$ are isomorphic and $\Phi(u)$ is also topologically stable due to Lemma~\ref{lemma:PolytopePolytope} and Lemma~\ref{lemma:PolytopeHyperplane}, the same arguments can be applied in order to show \[|\Po^-(u)|^\circ = K^\circ \cap \Phi(u)^{-1}((-\infty,0)).\]
   Observing that the restriction of $\varphi$ to the interiors $\varphi_{|\Po^-(p)|^\circ} \colon |\Po^-(p)|^\circ \to |\Po^-(u)|^\circ$ is a homeomorphism as well, we conclude that $K^\circ \cap \Phi(u)^{-1}((-\infty,0))$ and $K^\circ \cap \Phi(p)^{-1}((-\infty,0))$ are homeomorphic. 
    
    Let $F$ now be any face of $K$ with $\dim(F) \neq 0$, then by the same arguments it follows that $F^\circ \cap \Phi(u)^{-1}((-\infty,0))$ and $F^\circ \cap \Phi(p)^{-1}((-\infty,0))$ are homeomorphic. Furthermore, due to the fact that $\Phi(p)$ is topologically stable and the choice of $u$, if $\dim(F)=0$, the fact that \[F \subseteq K \cap \Phi(p)^{-1}((-\infty,0))\] implies that $F \subseteq K \cap \Phi(u)^{-1}((-\infty,0))$ and hence 
    \begin{align*}
        \partial K \cap \Phi(p)^{-1}((-\infty,0)) &= \left(\bigsqcup\limits_{\substack{F \preceq K, F \neq K \\ \dim(F) \neq 0}} F^\circ \sqcup \bigsqcup\limits_{F \in K_0} F\right) \cap \Phi(p)^{-1}((-\infty,0)) \\&\cong\left(\bigsqcup\limits_{\substack{F \preceq K, F \neq K \\ \dim(F) \neq 0}} F^\circ \sqcup \bigsqcup\limits_{F \in K_0} F\right) \cap \Phi(u)^{-1}((-\infty,0)) \\&= \partial K \cap \Phi(u)^{-1}((-\infty,0))
        \end{align*}
    Altogether, we conclude that $K \cap \Phi(p)^{-1}((-\infty,0))$ is homeomorphic to $K \cap \Phi(u)^{-1}((-\infty,0))$.
    \end{proof}

\section{Upper Bound}
\label{upperbound}
In this section we will provide a formal proof for the upper bounds on the Betti numbers of $F^{-1}((-\infty,0])$. For the sake of simplicity, we compute $\beta_k(F^{-1}((-\infty,0])$ using cellular homology. Ideally, we would like to equip $F^{-1}((-\infty,0])$ with a canonical CW-complex structure, i.e., the $k$-cells of the CW-complexes precisely correspond to the $k$-faces of the respective polyhedral complex, and attachment maps are given by face incidences (c.f.~Appendix~\ref{top_background}). However, $F^{-1}((-\infty,0])$ may contain unbounded polyhedra. In particular, an unbounded polyhedron cannot correspond to a CW-cell. To sidestep this issue, we construct a bounded polyhedral complex $\mathcal{Q}$ that is homotopy equivalent to $F^{-1}((-\infty,0])$.

The lineality space $L(P)$ of a polyhedron $P$ is defined as the vector space $V$ such that $p+v \in P$ for all $p \in P$ and $v \in V$. If $\Po$ is a complete $d$-dimensional polyhedral complex(i.e., $|\Po|=\R^d)$, then all polyhedra in $\Po$ have the same lineality space and hence the lineality space of the polyhedral complex is well-defined in this case.
\begin{lemma}
\label{lemma:linealityspace}
    Let $\Po$ be a $d$-dimensional polyhedral complex and let $\Po' \subseteq \Po$ be a subcomplex with $\#\Po'_k  \leq \binom{r}{d-k+1} $. Then there is a polyhedral complex $\mathcal{Q}$ such that \begin{enumerate}
        \item all polyhedra in $Q \in \mathcal{Q}$ contain a vertex,
        \item $|\mathcal{Q}|$ is a deformation retract of $|\Po' |$ and
        \item the number of $k$-faces $\#\mathcal{Q}_k$ is bounded by $\binom{r}{d-k-\ell+1}$, where $\ell$ is the dimension of the lineality space of $\Po$. 
    \end{enumerate} 
    \begin{proof}
    Let $V$ be the lineality space of $\Po$, $W \subseteq \R^d$ the subspace orthogonal to $V$ and $\pi \colon \R^d \to W$ the orthogonal projection. Then it holds that $\pi(P)$ is a face of $\pi(P')$ iff $ P \preceq P'$ for all $P,P' \in \Po'$ and $\mathcal{Q} = \{\pi(P) \mid P \in \Po'\}$ is a polyhedral complex. Furthermore, it holds that $\dim(P) = \dim(\pi(P)) + \ell$, where $\ell$ is the dimension of $V$ and therefore \[\#\mathcal{Q}_k = \#\Po'_{k+l} \leq \#\Po_{k+l} \leq \binom{r}{d-k-\ell+1}.\] Since $W$ is the lineality space of $\Po$ and hence also of $\Po'$, it follows that the map 
    \[R \colon |\Po'| \times  [0,1] \to |\mathcal{Q}|\]
    given by $R(w+v,t) = w +(1-t)v$ with $v \in V, w \in W, w+v \in |\Po'|$ is continuous and therefore a deformation retraction, proving the claim.    
    \end{proof}
\end{lemma}

\begin{lemma}\label{lemma:rpositive}
    Let $P \subseteq \R^d$ be a $d$-dimensional unbounded pointed polyhedron for some $d>0$. Then we have $P\cong \R^{d-1}\times [0,\infty)$. 
\end{lemma}
\begin{proof}
Without loss of generality, we may assume that $0\in P^\circ$  (else we apply translation). There exists a radius $r$ such that all vertices of the polyhedron $P$ are in the interior of the open $r$-ball $B_r(0)\coloneqq \{x: \|x\|_2< r\}$. We show that the space $P\cap B_r(0)$ is homeomorphic to $\R^{d-1}\times [0,\infty)$. Finally, we observe that $P\cap B_r(0)\cong P$ by scaling points in the polyhedra along the extreme rays.

Now consider $P\cap B_r(0)\subset D_r^d$ as a subset of the closed $d$-dimensional disk of radius $r$, i.e., $D_r^d\coloneqq \{x: \|x\|_2\leq  r\}$. Because $P$ and $B_r(0)$ are both convex, so is their intersection, therefore, radial projection from the origin $0\in P$, which fixes the origin and maps the boundary of $P$ to the boundary of $D^d_r$ is the desired homeomorphism.
\end{proof}

\begin{lemma}
\label{lemma:boundingcomplexes}
    Let $\Po$ be a finite polyhedral complex such that each $P\in \Po$ contains a vertex. Then $\Po$ is homotopy equivalent to a bounded polyhedral complex $\mathcal{Q}$ such that the number of $k$-faces $\#\mathcal{Q}_k$ is at most the number of $k$-faces $\#\mathcal{P}_k$.
\end{lemma}
\begin{proof}
    We prove the statement by induction on the number $\ell_\Po$ of unbounded faces of $\Po$. If $\ell_\Po=0$, we can pick $\mathcal{Q}=\Po$ because $\Po$ is bounded. 

    In the following, we construct a polyhedral complex $\mathcal{Q}'\simeq \Po$ with $\ell_{\mathcal{Q}'}=\ell_\Po-1$ and $\#\mathcal{Q}'_k\leq \#\mathcal{P}_k$. By induction hypothesis, we then obtain $\Po\simeq\mathcal{Q}'\simeq\mathcal{Q}$ and $\#\mathcal{Q}_k\leq \#\mathcal{Q}'_k\leq \#\mathcal{P}_k$ for each $k\in \mathbb{N}$, proving the statement.

    Let $P\in \Po$ be an unbounded $n$-dimensional polyhedron that is maximal with respect to inclusion. By Lemma~\ref{lemma:rpositive}, $P$ is homeomorphic to $\R^{n-1}\times [0,\infty)$. It is easy to observe that any homeomorphism (in particular the map $\phi\colon P\to \R^{n-1}\times[0,\infty)$ described in the proof of Lemma~\ref{lemma:rpositive}) maps the (topological) boundary $\partial P$ of $P$ precisely to the boundary $\R^{n-1}\times \{0\}\subset \R^{n-1} \times [0,\infty)$ of the codomain. Hence we have the following commutative diagram:

     \begin{center}
            \begin{tikzcd}
             P \arrow[hookleftarrow]{r}{\iota}\arrow[d, "\cong", "\phi"']& \partial P \arrow[d, "\cong", "\phi|_{\partial P}"']\\
        \R^{n-1}\times[0,\infty) \arrow{r}{\mathrm{pr}}  & \R^{n-1}\times \{0\}
    \end{tikzcd}
    \end{center}
    where $\iota\colon \partial P \hookrightarrow P$ is the canonical inclusion and $\mathrm{pr}$ is the canonical projection onto the first $d-1$ components.
    
    Since $P$ is maximal with respect to inclusion, $\mathcal{Q}'\coloneqq\Po\setminus\{P\}$ is a polyhedral complex with support $|\Po|\setminus \interior{P}$. Moreover, we have  $\ell_{\mathcal{Q}'}=\ell_\Po-1$ and $\#\mathcal{Q}'_k\leq \#\mathcal{P}_k$ for all $k\in \mathbb{N}$.
    
    Using $\Psi\coloneqq(\phi|_{\partial P})^{-1}\mathrm{pr}\circ \phi$, we define the following map $\Sigma\colon|\Po|\to |\Po|\setminus \interior{P}$:
    \begin{center}
        \[\Sigma(x)=\begin{cases}
            \Psi(x), &x\in P\\
            x, &\text{ else}
        \end{cases}\]
    \end{center}

    Note that this map is well-defined and continuous because one can easily observe that $\Psi(x)=x$ for all $x\in \partial P$. Moreover, it is a retraction, in particular a homotopy equivalence. 
\end{proof}

\begin{lemma}
\label{lemma:polyhedraltoCW}
    Let $\Po$ be a subcomplex of a $d$-dimensional complete polyhedral complex. Then there is a CW-complex $X$, such that \begin{enumerate}
        \item $|\Po|$ is homotopy equivalent to $X$ and 
        \item and the number of $k$-cells of $X$ is bounded by $\#\Po_{k+\ell}$,
    \end{enumerate}  
    where $\ell$ is the dimension of the lineality space of $\Po$.
\end{lemma}
\begin{proof}
    Follows by Lemma~\ref{lemma:linealityspace}, Lemma~\ref{lemma:boundingcomplexes} and the fact that bounded polyhedral complexes are canonically homeormorphic to CW-complexes(c.f. Appendix \ref{top_background}).
\end{proof}

\begin{lemma}
 \label{lemma:boundingbettinumberbyfaces}
    Let $C$ be a bounded $d$-dimensional polyhedral complex such that $|C|$ is contractible and $X$ a subcomplex of $C$. Then it holds that\[\beta_k(X) \leq \#\{(k+1)-\text{dimensional polyhedra in } C\setminus X\}\]
    for all $k \in [d-1]$.
\end{lemma}

\begin{proof}
    Since $C$ only contains bounded polyhedra, we can equip $C$ with a CW-complex structure (c.f. Appendix \ref{top_background}) and compute the Betti numbers using cellular homology.  
    
    Let $Y = C \setminus X$, then $Y$ is a set of polytopes, but not necessarily a polyhedral complex. By abuse of notation, we denote by $Y_k$ the $k$-dimensional polytopes in $Y$ and by $\sk_k(Y) = \{\ell-$dimensional polytopes in $Y \mid  \ell \leq k\}$. By the definition of cellular homology it holds that $\beta_k(\sk_{k+1}(X)) = \beta_k(X)$ and hence, if $\#Y_{k+1}=0$, then \[\beta_k(X) = \beta_k(\sk_{k+1}(C \setminus Y)) = \beta_k(\sk_{k+1}(C) \setminus (\sk_{k}(Y) \cup Y_{k+1})) \leq \beta_k(\sk_{k+1}(C)) = 0\] settling the induction hypothesis. To show the induction step, let $\sk_{k+1}(Y)= D \cup \hat{P}$, where $\hat{P}$ consists of one $(k+1)$-polytopes $P$ and all its faces in $Y$. Therefore, $\#Y_{k+1} = \#D_{k+1} +1$. Furthermore, let $B= C \setminus D$ and $A = C \setminus Y$. By the induction hypothesis we know that $\beta_k(B) \leq \#D_{k+1}$.

 Our goal is to embed the cellular homology group $H_{k}(B)$ into $\mathbb{Z}\oplus H_{k}(A)$. Such an embedding readily implies that $\beta_{k}(B)\leq 1+\beta_{k}(A)$. From this, the induction step follows:
\[\beta_k(B)\leq 1+\beta_k(A)\leq 1\] 

We first delete the $(k+1)$-dimensional polytope $P$ itself (that is, without deleting the redundant faces, resulting in a polyhedral complex whose support we denote by $B'$), and observe by the elaborate definition of cellular homology groups that this induces a map $\phi_1\colon H_{k}(B')\to \mathbb{Z}\oplus H_{k}(A)$. One can additionally observe that this map is an embedding: Notice that the homology arises from the (relative) homologies of the chain complex
\[\ldots \rightarrow C_{k+1}\xrightarrow{\partial_{k+1}} C_k \xrightarrow{\partial_k} C_{k-1}\rightarrow \ldots \]
Deleting $P$ decreases the image of the boundary map $\partial_{k+1}$ and hence increases the $k$-th homology; however, it is straightforward to observe that
\[\phi_1\colon H_{k}(B')\cong \altfrac{\ker\partial_k}{\im\partial_{k+1}}\to [\sigma]\oplus \altfrac{\ker\partial_k}{\im\partial_{k+1}\cup [\sigma]}\cong \mathbb{Z}\oplus H_{k}(A)\]

which maps a homology class $[\sum_{\tau\in C_k} c_{\tau}\tau]$ from the domain to $(c_{\sigma},[\sum_{\tau\in \altfrac{C_k}{[\sigma]}} c_{\tau}\tau])$ is an embedding, where $\sigma$ is the generator corresponding to $P$.

To finish off the construction of the embedding, we finally define $\phi_2\colon H_k(B)\to H_k(B')$, i.e., the map induced by deleting all faces of $P$ in $\hat{P}$. This operation might reduce the kernel of the $k$-th boundary map, and hence potentially decrease the $k$-th Betti number. It is, however, again straightforward to observe that the map is injective in any case, in a similar fashion as above.

The claimed embedding is now $\phi_1\circ\phi_2$, finishing the proof for $k \in [d-1]$.
\end{proof}

\upperbound*

\begin{proof}
Theorem 1 in \citet{Bounding_Serra} states that $F$ has at most $r \coloneqq \sum_{(j_1,\ldots,j_{L})\in J} \prod_{l=1}^{L} \binom{n_l}{j_l}$ linear regions.
    Let $\Po$ be the canonical polyhedral complex of $F$. Since $\Po$ is complete, it follows that $\#\Po_k \leq \binom{r}{d-k+1}$. Furthermore, for $P \in \Po$ it holds that $F_{\mid P}$ is affine linear and hence $F_{\mid P}^{-1}((-\infty,0])$ is a half-space and therefore $P \cap F^{-1}((-\infty,0]) = P \cap F_{\mid P}^{-1}((-\infty,0])$ is a polyhedron. We define $\Po^- =\{P \cap F^{-1}((-\infty,0]) \mid P \in \Po\}$ and $ \Po^+ = \{P \cap F^{-1}([0,\infty)) \mid P \in \Po\}$, which due to the continuity of $F$ are subcomplexes of the refinement of $\Po$ where $F$ takes on exclusively non-positive respectively non-negative values on all polyhedra. It follows immediately from the definition that $\#\Po^+_k \leq \#\Po_k \leq \binom{r}{d-k+1}$. Let $s$ be the dimension of the lineality space of the complete polyhedral complex $\Po^- \cup \Po^+$.  By Lemma~\ref{lemma:polyhedraltoCW} we obtain a CW-complex $C$ that is homotopy equivalent to $\Po^- \cup \Po^+$, and therefore in particular contractible. Furthermore, since $\Po^-$ and $\Po^-$ are subcomplex of $\Po^- \cup \Po^+$, we obtain subcomplexes $X$ and $Y$ of $C$ such that $X$ is homotopy equivalent to $F^{-1}((-\infty,0])$ and $Y$ is homotopy equivalent to $F^{-1}([0,\infty))$. It clearly holds that $C\setminus X \subseteq Y$ and hence by Lemma~\ref{lemma:boundingbettinumberbyfaces} it follows that \[\beta_k(F^{-1}((-\infty,0])) = \beta_k(X) \leq \#\{(k+1)\text{-dimensional polyhedra in }(C \setminus X)\} \leq \binom{r}{d-k-s}, \]
    proving the claim.
\end{proof}
\end{document}